\documentclass[twoside]{article}

\usepackage[accepted]{aistats2025}
\usepackage{tikz}
\usepackage[colorlinks=true, linkcolor=blue, urlcolor=blue, citecolor=blue]{hyperref}
\usepackage{url}     
\usepackage{amsmath,amsfonts,mathtools,amssymb, amsthm}
\usepackage{nicefrac}       
\usepackage{microtype}      
\usepackage{xcolor}         
\usepackage{comment}
\usepackage{etoolbox}
\usepackage{ifthen}
\usepackage{aligned-overset}
\usepackage{algorithm}
\usepackage[noend]{algpseudocode}
\usepackage{siunitx}
\DeclareSIUnit{\byte}{B}
\usepackage{todonotes}
\usepackage{mathrsfs}
\usetikzlibrary{positioning}
\usepackage{graphicx}
\usepackage{xcolor}
\usepackage{tikz}
\usepackage[shortlabels]{enumitem}
\usepackage{pgfplots}
\usepackage{wrapfig}
\usepgfplotslibrary{groupplots}
\usetikzlibrary{external}
\usepackage{float}
\usepackage{subcaption}

\usepackage[round]{natbib}

\usepackage{comment}
\newcommand{\pacsbo}{\textsc{Pacsbo}}

\newcommand{\safeopt}{\textsc{SafeOpt}}

\newcommand{\ie}{i\/.\/e\/.,\/~}
\newcommand{\eg}{e\/.\/g\/.,\/~}
\newtheorem{lemma}{\bf Lemma}
\newtheorem{theorem}{\bf Theorem}

\newtheorem{corollary}{\bf Corollary}

\newtheorem{assumption}{\bf Assumption}

\newtheorem{remark}{\bf Remark}
\newcommand{\domain}{\mathcal A}

\newcommand{\underover}[2]{\underset{#1}{\underbrace{#2}}}

\newcommand{\fakepar}[1]{\vspace{1mm}\noindent\textbf{#1.}}
\newcommand{\C}{\mathrm{C}}
\DeclareMathOperator*{\argmax}{arg\,max}

  \definecolor{aaltoRed}{RGB}{239,51,64}%
  \definecolor{aaltoBlue}{RGB}{0,94,184}%
  \definecolor{magenta}{RGB}{255,0,255}
  \definecolor{gray}{RGB}{128,128,128}
\definecolor{orange}{RGB}{255,165,0}
\definecolor{purple}{RGB}{128,0,128}

\begin{document}

%

%

\twocolumn[

\aistatstitle{Safe exploration in reproducing kernel Hilbert spaces}

\aistatsauthor{Abdullah Tokmak \And Kiran G.\ Krishnan \And  Thomas B.\ Schön \And Dominik Baumann}

\aistatsaddress{\footnotesize{Department of Electrical} \\ \footnotesize{Engineering and Automation} \\ Aalto University \And  \footnotesize{Department of Electrical} \\ \footnotesize{Engineering and Automation} \\ Aalto University  \And \footnotesize{Department of} \\ \footnotesize{Information Technology} \\ Uppsala University  \And  \footnotesize{Department of Electrical} \\ \footnotesize{Engineering and Automation} \\ Aalto University}
]

\begin{abstract}
Popular safe Bayesian optimization (BO) algorithms learn control policies for safety-critical systems in unknown environments. However, most algorithms make a smoothness assumption, which is encoded by a known bounded norm in a reproducing kernel Hilbert space (RKHS). The RKHS is a potentially infinite-dimensional space, and it remains unclear how to reliably obtain the RKHS norm of an unknown function. In this work, we propose a safe BO algorithm capable of estimating the RKHS norm from data. We provide statistical guarantees on the RKHS norm estimation, integrate the estimated RKHS norm into existing confidence intervals and show that we retain theoretical guarantees, and prove safety of the resulting safe BO algorithm. We apply our algorithm to safely optimize reinforcement learning policies on physics simulators and on a real inverted pendulum, demonstrating improved performance, safety, and scalability compared to the state-of-the-art.
\end{abstract}
\addtocontents{toc}{\protect\setcounter{tocdepth}{0}}
\section{INTRODUCTION}\label{sec:intro}

\begin{figure*}
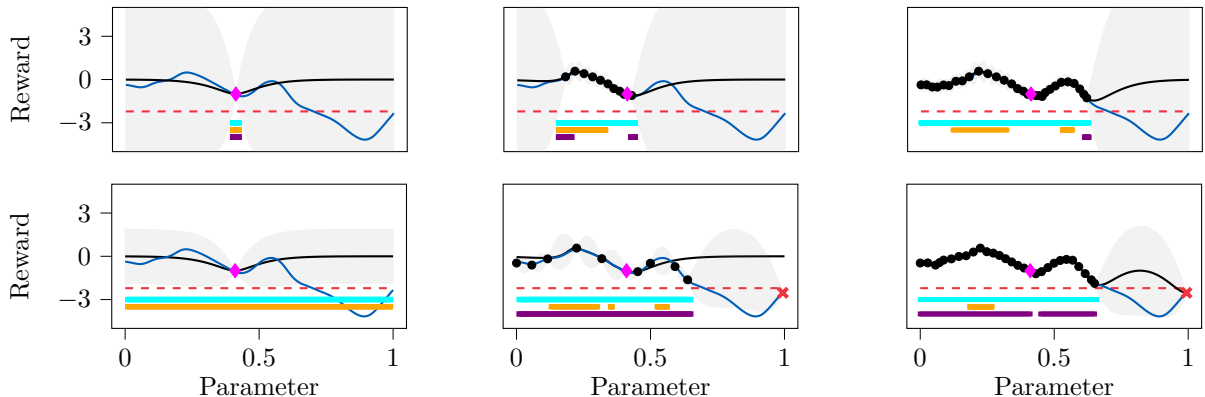

    \centering
    \begin{subfigure}{6cm}
    \centering
    \input{final_figs/SafeOpt_intro/True-Iteration-0}
    \end{subfigure}
    \begin{subfigure}{5.25cm}
    \centering
    \input{final_figs/SafeOpt_intro/True-Iteration-10}
    \end{subfigure}
    \begin{subfigure}{5.25cm}
    \centering
    \input{final_figs/SafeOpt_intro/True-Iteration-last}
    \end{subfigure}
    \begin{subfigure}{6cm}
    \centering
    \input{final_figs/SafeOpt_intro/Under-estimation-0}
    \end{subfigure}
    \begin{subfigure}{5.25cm}
    \centering
    \input{final_figs/SafeOpt_intro/Under-estimation-10}
    \end{subfigure}
    \begin{subfigure}{5.25cm}
    \centering
    \input{final_figs/SafeOpt_intro/Under-estimation-last}
    \end{subfigure}
    \caption{\emph{Toy example of safe BO and the influence of the RKHS norm.} 
    We aim to maximize the reward function (blue) while only sampling above the safety threshold (red dashed line).
    The predicted function (black line) is computed based on iteratively acquired samples (black dots, initial sample in magenta), and the gray shaded area shows the confidence intervals.
    At each iteration, we compute a set of parameters that we believe to be safe (cyan), potential expanders (purple), and potential maximizers (orange), thus safely balancing exploration and exploitation.    
    The upper sub-figures show safe BO, where the true RKHS norm is used to compute the confidence intervals, while the lower sub-figures are generated with an under-estimated RKHS norm.
An under-estimation of the RKHS norm can yield confidence intervals that do not contain the reward function, which may eventually lead to unsafe experiments (red cross). 
}
    \label{fig:SafeOpt intro}
\end{figure*}

When learning policies for systems that act in the real world, such as mobile robots or autonomous vehicles, two crucial requirements must be met: \emph{(i)} the learning algorithms we use must be sample efficient, as learning experiments are time-consuming and cause wear and tear to the hardware; and \emph{(ii)} we must guarantee safety during exploration, \ie while testing new policies, for systems not to damage themselves, their environment, or endanger people.
Currently, one of the most popular tools for policy learning is reinforcement learning (RL).
Without the need for a dynamics model, RL learns a policy through trial and error, \ie by performing experiments and receiving a reward signal in return that it tries to maximize.
Unfortunately, RL struggles with both requirements.
Hence, the most impressive results of RL algorithms have been achieved in simulated or gaming environments 
\citep{Lillicrap2019Continuous}.

An alternative to RL for policy learning is combining Bayesian optimization (BO) \citep{garnett2023bayesian} with Gaussian process (GP) regression \citep{Rasmussen2006Gaussian}.
When modeling the reward function with a GP, we can leverage this model and pose the decision of where to explore next as an optimization problem.
This way of sequential decision-making drastically improves sample efficiency, as shown in numerous hardware experiments \citep{antonova2017deep, calandra2016bayesian, marco2016automatic}.
Thus, combining GPs and BO meets the first requirement.
For the second requirement, safe BO algorithms guarantee safety during exploration with high probability; a well-known example is \safeopt\ \citep{sui2015safe}.
\safeopt, as well as other popular safe BO algorithms, assume that the reward function lies in a reproducing kernel Hilbert space (RKHS). 
Moreover, guaranteeing safety requires an additional smoothness assumption, which is encoded by knowing an upper bound on the norm of the reward function in that RKHS.
Even though the assumption elegantly paves the way to guarantee safety with high probability, it is highly unrealistic since the RKHS is a potentially infinite-dimensional space, and it is unclear how to guess that upper bound for unknown reward functions.
If we incorrectly specify the RKHS norm, \ie if the true RKHS norm is larger than the bound we assume, safety guarantees may become obsolete, as we illustrate in Figure~\ref{fig:SafeOpt intro}.

\paragraph{Contribution}
In response, we present a data-driven approach to compute an RKHS norm over-estimation with statistical guarantees.
We integrate this RKHS norm over-estimation into a safe BO algorithm reminiscent of \safeopt, for which we prove safety with high probability.
Moreover, we extend our safe BO algorithm by introducing a notion of locality.
By considering local RKHS norms, which are potentially smaller than the global RKHS norm, we can explore more optimistically and significantly improve scalability by separately discretizing local sub-domains.
We compare our algorithm with \safeopt\ in a synthetic example and challenging robotic simulation benchmarks, where we demonstrate the benefits of over-estimating the RKHS norm from data instead of randomly guessing it.
Finally, we demonstrate the applicability of our algorithm to real-world systems in a hardware experiment.\footnote{%
See
\url{https://github.com/tokmaka1/AISTATS_2025} for the code and \url{https://safeexploration.wordpress.com/} for videos of the experiments.
}

\section{PROBLEM SETTING AND PRELIMINARIES}\label{sec:preliminaries}
We cast safe policy search as a constrained optimization problem, where the objective function quantifies the performance. 
We consider parameterized policies. The parameters, which could be controller parameters, serve as the decision variables of the optimization problem.

\paragraph{Problem setting}
We aim to maximize an unknown reward function $f{:}\;\domain\subseteq\mathbb{R}^n\to\mathbb{R}$ while guaranteeing safety.
We define safety as only sampling parameters~$a\in\domain$ corresponding to reward values larger than a pre-defined safety threshold~$h\in\mathbb R.$
Thus, we write the optimization problem as
\begin{align} \label{eq:opt}
    \max_{a \in \domain} f(a) \quad \text{subject to } \; f(a_t)\geq h, \; \forall t \geq 1.
\end{align}
We solve~\eqref{eq:opt} by sequentially querying the reward function at each iteration~$t\in\mathbb N$.
In return, we receive measurements
 $y_t\coloneqq f(a_t)+\epsilon_t$, where~$\epsilon_t$ is independent and identically distributed (i.i.d.)  $\sigma$-sub-Gaussian measurement noise.
We denote the queried parametrizations until iteration~$t$ by~$a_{1:t}\coloneqq[a_1, \ldots, a_{\mathrm t}]^\top$ and the corresponding measurements by~$y_{1:t}$. 
GP regression provides a natural tool to estimate~$f$, as done in \safeopt\ \citep{sui2015safe} and other BO algorithms \citep{srinivas2012information}.
Given data~$a_{1:t}$ and~$y_{1:t}$ at each iteration~$t$, the posterior GP mean and variance are
\begin{align} 
        \mu_t(a) &= k_t(a)^\top(K_t+\sigma^2 I_t)^{-1}y_{1:t}, \label{eq:mean}\\
    \sigma_t^2(a) &= k(a,a) - k_t(a)^\top(K_t+\sigma^2 I_t)^{-1}k_t(a), \label{eq:covariance}
\end{align}
respectively \citep{Rasmussen2006Gaussian}, where~$k(a,a)$ is the kernel evaluated at~$a\in\domain$,~$k_t(a)=[k(a,a_1), \ldots, k(a,a_{\mathrm{t}})]^\top\in\mathbb{R}^t$ the covariance vector,~$K_t\in\mathbb{R}^{t\times t}$ the covariance matrix with entry~$k(a_i,a_j)$ at row~$i$ and column~$j$ for all~$i,j\in\{1,\ldots,\mathrm{t}\}$, and~$I_t$ the $t\times t$ identity matrix.
Similar to \safeopt\ and other safe BO algorithms, we assume that the reward function lies in the RKHS of kernel~$k$, \ie $f\in H_k$.
This assumption is, in general, non-restricting since many kernels satisfy the universal approximation property \citep{micchelli2006universal}, \ie even if~$f\not\in H_k$, there exists a $\tilde f \in H_k$ such that $\sup_{a\in\domain}\lvert f(a)-\tilde f(a)\rvert < \epsilon$ for all $\epsilon>0$.
Given this assumption, we can obtain frequentist confidence intervals~$Q_t(a)$ around the posterior mean~$\mu_t$ that contain~$f$ with high probability, \ie 
\begin{align}\label{eq:Q}
    Q_t(a)&\coloneqq   \mu_{t}(a)\pm \beta_t \sigma_t(a), \\
    \beta_t &= B_t+\sqrt{\sigma
\log\det(I_t+K_t/\sigma)
-2\sigma\log(\delta)}, \nonumber
\end{align}
with confidence parameter~$\delta\in(0,1)$.
%
We obtain~\eqref{eq:Q} by combining Theorem~3.11 by \cite{abbasi2013online} with Remark~3.13 by \cite{abbasi2013online} as detailed in Appendix~\ref{app:abbasi}.
In~\eqref{eq:Q},~$B_t$ is an over-estimation of the ground truth RKHS norm, \ie $B_t\geq \|f\|_k = \sqrt{\sum_{s=1}^\infty\sum_{t=1}^\infty \alpha_s\alpha_t k(x_s,x_t)}$, where~$\alpha$ are the coefficients and~$x$ are the center points of the RKHS function~$f$ (see Appendix~\ref{app:RKHS_norm_derivation} for the derivation).
Notably, an under-estimation of the RKHS norm might lead to unsafe experiments (see Figure~\ref{fig:SafeOpt intro}), while a too conservative over-estimation might yield too cautious exploration and even premature stopping (see Figure~\ref{fig:intro_over-estimation} in Appendix~\ref{app:intro}). 
In this paper, we 
compute a data-dependent~$B_t$ at each iteration~$t$ that over-estimates~$\|f\|_k$ with high probability.
The data-driven RKHS norm over-estimation is the chief distinction between our approach and other safe BO algorithms like \safeopt\ that \emph{guess the RKHS norm a priori}. 

\paragraph{Lipschitz constant}
Besides knowing an upper bound on the RKHS norm, safe BO algorithms like \safeopt\ typically assume a known  Lipschitz constant.
We replace the Lipschitz constant with an RKHS norm induced continuity  using the  kernel (semi)metric 
\begin{align}\label{eq:semi-metric}
  \hspace*{-0.5em} d_k(a,a^\prime)^2= k(a,a)+k(a^\prime, a^\prime)-k(a,a^\prime)-k(a^\prime,a),
\end{align}
which we derive in Lemma~\ref{le:cont} in Appendix~\ref{proof:safety}.

\paragraph{Safe exploration}
Equivalent to \safeopt, we define the contained set~$C_t(a)\coloneqq C_{t-1}(a)\cap Q_t(a),\, C_0=\mathbb R$, lower bound $\ell_t(a)\coloneqq\min C_t(a)$, and upper bound~$u_t(a)\coloneqq\max C_t(a)$
to quantify probabilistically whether a policy parameter~$a$ is safe.
At each iteration~$t$, we restrict function evaluations to a safe set~$S_t\subseteq \domain$ that only contains parameters~$a$ that are safe with high probability:
\begin{align} \label{eq:safe_set}
    S_t \coloneqq \cup_{a\in S_{t-1}}
    \{
a^\prime \in \domain \vert \ell_t(a)- B_t d_k(a,a^\prime) \geq h
    \}.
\end{align}
To start exploration, we assume that a set of initial safe samples~$\emptyset\neq S_0\subseteq\domain$ is given. 
Moreover, we define
\begin{align}
M_t&\coloneqq \{a\in S_t \vert u_t(a)\geq \max_{a^\prime\in S_t} \ell_t(a^\prime)\}, \label{eq:maximizer}\\
G_t&\coloneqq\{a\in S_t\vert g_t(a)>0\},\label{eq:expander} \\ g_t (a) &= \mathrm{card}(a^\prime \in \domain\setminus S_t \vert u_t(a)-B_td_k(a,a^\prime)\geq h), \nonumber
\end{align}
as the set of potential safe maximizers and potential safe expanders, respectively.
At each iteration~$t$, 
the next parameter~$a_{t+1}$ is given by the most uncertain parameter within~$M_t\cup G_t$, \ie
\begin{align}\label{eq:acquisition}
    a_{t+1}= \argmax_{a\in M_t\cup G_t} \beta_t \sigma_{t}(a),
\end{align}
which results in safely balancing exploration and exploitation to solve~\eqref{eq:opt}.

\section{SAFE BO WITH RKHS NORM OVER-ESTIMATION}
Algorithm~\ref{alg:global_safe_BO} summarizes the proposed safe BO algorithm with the RKHS norm over-estimation.
In each iteration, we 
determine the next sample with which we conduct a new experiment.
The sample acquisition is described in Algorithm~\ref{alg:acquisition}.
First, we define the GP model given the current set of samples.
Then, we compute an over-estimation of the RKHS norm by querying Algorithm~\ref{alg:PAC}, which we extensively explain in Section~\ref{sec:RKHS}.
Moreover, we compute the confidence intervals, the set of safe samples~$S_t$, the set of potential maximizers~$M_t$, and the set of potential expanders~$G_t$.
Finally, we return the most uncertain parameter within~$M_t\cup G_t$ and its corresponding uncertainty.
The acquisition function is reminiscent of \safeopt\ with the crux difference lying in the RKHS norm~$B_t$ (l.~2), where \safeopt\ \emph{guesses the RKHS norm a priori} and \emph{maintains that guess}.
Hence, we naturally recover \safeopt\ by replacing the query of Algorithm~\ref{alg:PAC} with an oracle.

\begin{algorithm}
\begin{algorithmic}[1]
\Require $k$, $\domain$, $S_0$, $\delta$, $\kappa$, $\gamma$ $m$, $\sigma$
\State Init: $a_{1}, y_{1}$ samples corresponding to~$S_0$, $B_0=\infty$
\For{$t=1,2,\ldots$}
\State $a_{t+1}\gets $  Algorithm~\ref{alg:acquisition}$(k, \domain, S_{t-1}, \delta, \kappa, \gamma, m, \sigma, t)$
\State $y_{t+1} \gets  f(a_{t+1})+\epsilon_{t+1}$ \Comment{Conduct experiment}
\EndFor
\State \Return Best safely evaluable parameter~$a\in\domain$
\end{algorithmic}
\caption{Proposed safe BO algorithm with RKHS norm over-estimation.}
\label{alg:global_safe_BO}
\end{algorithm}

\begin{algorithm}
\begin{algorithmic}[1]
\Require $k$, $\domain$, $S_{t-1}$, $\delta$, $\kappa$, $\gamma$, $t$, $B_{t-1}$, $t$, $\sigma$
\State Compute~$\mu_t$ and~$\sigma_{t}^2$ given~$a_{1:t}, y_{1:t}$ \Comment{\eqref{eq:mean}, \eqref{eq:covariance}}
\State $B_t \gets$ Algorithm~\ref{alg:PAC}$(\gamma, \kappa, m, \domain, k, B_{t-1}, a_{1:t}, y_{1:t}, k)$
\State Compute sets~$Q_t(a)$, $C_t(a)$, and bounds~$u_t(a),\ell_t(a)$ from samples~$a_{1:t}, y_{1:t}$, and~$B_t$
\Comment{\eqref{eq:Q}}
\If{$t>1$} compute $S_t$~\eqref{eq:safe_set} \textbf{else} $S_t\gets S_0$
\Comment{\eqref{eq:safe_set}}
\EndIf
\State Compute $\omega_t \coloneqq \beta_t\sigma_{t}$ and $M_t,G_t$ \Comment{\eqref{eq:maximizer}, \eqref{eq:expander}}
\State \Return $\underset{a\in M_t\cup G_t}{\argmax}\omega_t(a)$, $ \underset{a\in M_t\cup G_t}{\max}\omega_t(a)$ \Comment{\eqref{eq:acquisition}}
\caption{Sample acquisition.}
\label{alg:acquisition}
\end{algorithmic}
\end{algorithm}

In the remainder of this section, we
present the RKHS norm over-estimation to compute~$B_t$ (Section~\ref{sec:RKHS}), provide theoretical guarantees for~$B_t\geq \|f\|_k$, integrate the estimated RKHS norm into existing confidence intervals and prove safety of Algorithm~\ref{alg:global_safe_BO} (Section~\ref{sec:theory}), and extend Algorithm~\ref{alg:global_safe_BO} by introducing a notion of locality (Section~\ref{sec:locality}).

\subsection{RKHS norm over-estimation} \label{sec:RKHS}

The RKHS norm over-estimation used in Algorithm~\ref{alg:acquisition} is based on two pillars: 
\emph{(i)} a recurrent neural network (RNN) \citep{hochreiter1997long} that predicts the RKHS norm for each iteration, and \emph{(ii)} random RKHS functions that simulate the potential behavior of the unknown reward function~$f$.

\paragraph{RNN}
We use an RNN to estimate the RKHS norm~$\|f\|_k$ based on the current samples~$a_{1:t}$ and~$y_{1:t}$.
Specifically, for each iteration, we compute the RKHS norm of the GP mean function~$\|\mu_t\|_k$ and the reciprocal integral of the posterior variance~$\sigma_t^2$, which quantifies sampling density, and store them as sequences.
As the sampling density increases, the GP mean~$\mu_t$ and its RKHS norm~$\|\mu_k\|_k$ approximate the reward function~$f$ and its RKHS norm~$\|f\|_k$  more closely. 
While the two sequences serve as the input to the RNN, we also require labels to train it.
To this end, we optimize artificial RKHS functions~$g\in H_k$, whose known RKHS norms~$\|g\|_k$ serve as the labels 
for training the RNN, using our proposed safe BO algorithm.
We provide more details on the RNN in Appendix~\ref{app:RNN}, including its architecture, the generation of training data, its performance, and its role while executing the algorithm.
We want to highlight that the RNN merely provides a \emph{heuristic lower bound} on the estimated RKHS norm and solely acts as an additional layer of conservatism (see~\eqref{eq:impossible} in Appendix~\ref{proof:RKHS_scenario}); the hereafter introduced guarantees are \emph{independent} from the estimation of the RNN.
The RNN could be replaced by different function approximators; we choose an RNN to exploit the sequential nature of the inputs.

\paragraph{Random RKHS functions}
The second pillar is the computation of random RKHS functions with which we obtain theoretical guarantees on the RKHS norm over-estimation. 
In essence, the random RKHS functions~$\rho_j\in H_k, \, j \in\{1,\ldots,\mathrm m\}$ capture the behavior of the unknown reward function~$f$, as shown in Figure~\ref{fig:random_RKHS_functions}. 
Ideally, we would create random RKHS functions that capture the entire RKHS; however, this would require computing infinite sums.
Hence, in implementation, we follow a pre-RKHS approach as described in Appendix~C.1 by \cite{Fiedler2021Practical} to create random RKHS functions~$\rho_j=\sum_{s=1}^{\hat N} \alpha_s k(\cdot, x_s),\, \hat N \gg t$.
We require the random RKHS functions to interpolate the given samples~$y_{1:t}$ subject to~$\sigma$-sub-Gaussian noise.
Thus, the interpolating property determines the first~$\alpha_{1:t}$ coefficients.
Moreover, we assume that the first center points~$x_{1:t}$ are equal to the parameters~$a_{1:t}$.  
The remaining ~$\alpha_{t+1:\hat N}, x_{t+1:\hat N}$ are i.i.d.\ samples from uniform distributions with $x\in\domain$ and $a\in[-\bar\alpha, \bar\alpha]$, introducing the required stochasticity.
Subsequently, the random RKHS functions exhibit vastly different behavior for fewer samples and approach~$f$ for more samples (see Figure~\ref{fig:random_RKHS_functions}), which will yield tighter RKHS norm over-estimations for an increasing sample density.

\begin{figure*}
    \centering
    \input{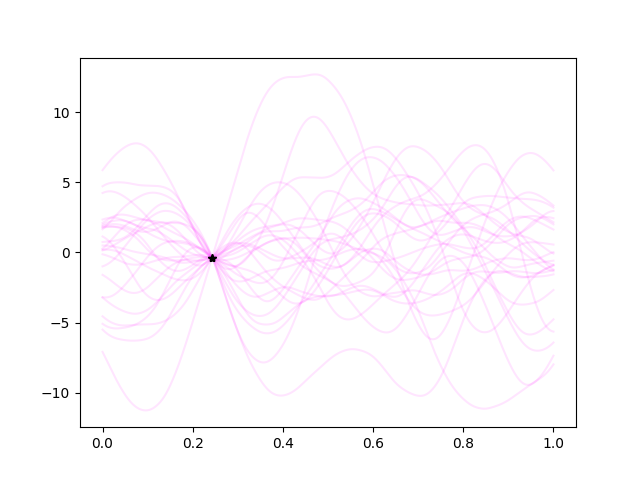}
    \input{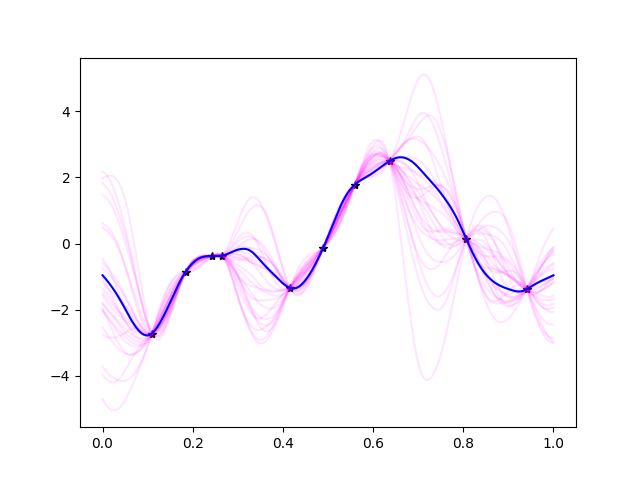}
    \input{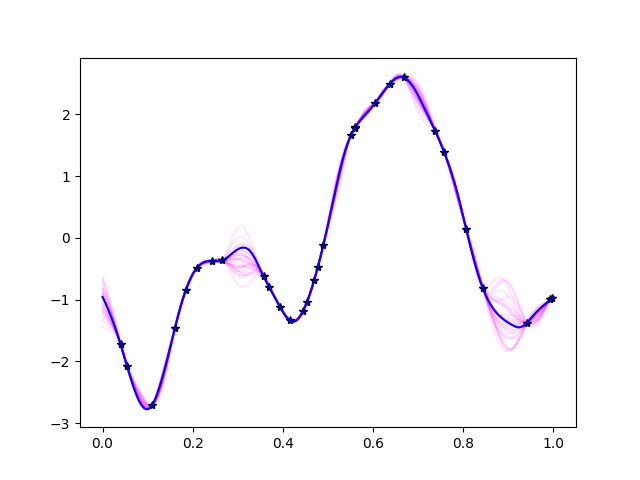}
    \caption{\emph{Random RKHS functions.}
    The random RKHS functions approach the unknown reward function with more samples.
    We generated the plots with the Matérn32 kernel with~$\ell=0.1$.
    The remaining hyperparameters were $\hat N=500$, $\bar\alpha=1$, and $\sigma=10^{-2}$.
    The reward function~$f$ has 1000 random center points and coefficients, which we scale to yield~$\|f\|_k=5$.
    We sampled the parameters~$a_{1:t}\subseteq \domain$ from a uniform distribution.
    }
    \label{fig:random_RKHS_functions}
\end{figure*}

\begin{algorithm}
\begin{algorithmic}[1]    
\Require $\gamma$, $\kappa$, $m$, $\domain$, $k$, $B_{t-1}$, $a_{1:t}$, $y_{1:t}$, $k$
\State $B_t \gets$ RKHS norm estimation given~$a_{1:t},y_{1:t}, k, \domain$ with RNN
\State Construct~$m$ random RKHS 
functions~$H_k \ni \rho_{t,j}{:}\;\domain\rightarrow\mathbb R$, with~$\|\rho_{t,j}\|_k$ given $a_{1:t},y_{1:t}$
\State Sort functions by ascending RKHS norm~$\{\rho_{t,j}\}_{j=1}^m$
\If{$B_t<\|\rho_{t,m}\|_k$}
\Statex $r\gets \max_{r\in\{1,\ldots, m-1\}} r \quad \text{subject to}$
\Statex $\sum^{r}_{i=0}{m \choose i} \gamma^i (1-\gamma)^{m-i}\leq\kappa\;\land\;  B_t  < \|\rho_{t,m-r}\|_k$
\Statex $B_t \gets \|\rho_{t,m-r}\|_k$
\EndIf
\If{$B_t < B_{t-1}$}
$B_t\gets B_{t-1}$
\EndIf
\State \Return $B_t$
\end{algorithmic}
\caption{RKHS norm over-estimation}
\label{alg:PAC}
\end{algorithm}

\paragraph{Algorithm} The RKHS norm over-estimation is summarized in Algorithm~\ref{alg:PAC}. 
First, we receive the RKHS norm estimation from the RNN given the current set of samples.
Second, we construct~$m$ i.i.d.\ random RKHS functions with known RKHS norms.
Based on the return of the RNN and the RKHS norms of the random RKHS functions, we return~$B_t$, which over-estimates~$\|f\|_k$ with high probability.  
The explicit form of~$B_t$ becomes clear in Theorem~\ref{th:RKHS_scenario}.

\begin{remark}\label{re:limit}
Our design choices decrease the space that is covered by the random RKHS functions.
Nevertheless, the random RKHS functions in Figure~\ref{fig:random_RKHS_functions} display a high degree of randomness, although they lie in a sub-space of the pre-RKHS from which $f$ is generated.
This supports the design choices.
An alternative to the pre-RKHS approach is to work with orthonormal bases of RKHSs provided in Theorem~4.38 by \cite{Steinwart2008SVM} for the squared-exponential kernel and by \cite{tronarp2024orthonormal} for other translation-invariant kernels.
\end{remark}

\begin{remark}
Although we integrate the RKHS norm over-estimation into \safeopt, it applies equally to any extension such as by \cite{berkenkamp2023bayesian, sui2018stagewise, Bhavi2023GSO}.
Besides, the relevance of the RKHS norm goes beyond BO. 
It appears in, \eg statistics \citep{gretton2012kernel} or kernel-based function approximation \citep{maddalena2021deterministic}.
\end{remark}

\subsection{Theoretical analysis}\label{sec:theory}
In the following, we present theoretical guarantees for the RKHS norm over-estimation and Algorithm~\ref{alg:global_safe_BO}.
First, we make an assumption on the inputs, the noise, and the kernel, akin to \cite{chowdhury2017kernelized}.
\begin{assumption}\label{asm:chow}
The kernel~$k{:}\; \mathbb R \times \mathbb R \rightarrow \mathbb R_{\geq 0}$ is symmetric, positive definite, and continuous.
Moreover, the action sequence~$\{a_t\}_{t=1}^\infty$ is an~$\mathbb{R}^n$-valued discrete time stochastic process
    and~$a_t$ is $\mathcal F_{t-1}$-measurable $\forall t\geq 1$.
The noise~$\{\epsilon_t\}_{t=1}^\infty$ is a real-valued stochastic process and for some~$\sigma\geq 0$ and all~$t\geq 1$, $\epsilon_t$ is \emph{(i)}~$\mathcal{F}_t$-measurable and \emph{(ii)}~$\sigma$-sub-Gaussian conditionally on~$\mathcal F_{t-1}$.
\end{assumption}

Next, we connect the RKHS norms of the random RKHS functions and the reward function.
\begin{assumption}\label{asm:random}
For any iteration~$t\geq 1$, given~$a_{1:t}, y_{1:t}$, the RKHS norms of the random RKHS functions~$\|\rho_{t,j}\|_k, j \in \{1, \ldots, \mathrm m\}$, and the RKHS norm of the reward function~$\|f\|_k$ are i.i.d.\ samples from the same---potentially unknown---probability space.
\end{assumption}


We discuss Assumption~\ref{asm:random} in Section~\ref{sec:limitations} and in Appendix~\ref{app:assumption_rebuttal}.
The following theorem is our main theoretical contribution and   proves~$B_t\geq\|f\|_k $ with high probability.
Specifically, it shows that~$B_t\geq\|f\|_k$ is probably approximately correct (PAC) \citep{Shwartz2014Understanding}. 
\begin{theorem} [RKHS norm over-estimation]\label{th:RKHS_scenario}
Given Assumptions~\ref{asm:chow} and~\ref{asm:random}, for any iteration~$t\geq 1$,~$\gamma,\kappa\in(0,1)$, and $m\in\mathbb{N}$ such that $(1-\gamma)^{m-1}(1+\gamma(m-1))\leq\kappa$, consider~$B_t$ returned by Algorithm~\ref{alg:PAC}.
With confidence at least~$1-\kappa$, we have~$B_t\geq \|f\|_k$ with probability at least~$1-\gamma$.
\end{theorem}
\begin{proof}(Idea)
We formulate the RKHS norm over-estimation as a chance-constrained optimization problem, which we solve using a sampling-and-discarding scenario approach. 
We obtain PAC bounds by leveraging Theorem~2.1 by \cite{campi2011sampling}.
We provide a detailed proof in Appendix~\ref{proof:RKHS_scenario}.    
\end{proof}

The following corollary lifts Theorem~\ref{th:RKHS_scenario} to hold jointly for all iterations~$t\geq 1.$
\begin{corollary}[Lifting Theorem~\ref{th:RKHS_scenario} to all iterations]\label{co:RKHS_stop}
Under the assumptions of Theorem~\ref{th:RKHS_scenario}, receive~$B_t$ from Algorithm~\ref{alg:PAC} at all iterations~$t$.
Then, with confidence at least~$1-\kappa$,~$B_t$ over-estimates the ground truth RKHS norm~$\|f\|_k$ jointly for all iterations~$t\geq 1$ with probability at least~$1-\gamma$.
\end{corollary}
\begin{proof}(Idea)
First, we show that the discrete-time stochastic process~$\{B_t\}_{t=1}^T, T\in\mathbb N$, containing the PAC RKHS norms is a supermartingale.
Then, we use a standard stopping time criterion construction as in Theorem~1 by \cite{Abbasi2011Improved}.
We provide a detailed proof in Appendix~\ref{proof:RKHS_stop}.
\end{proof}

Next, we integrate the RKHS norm over-estimation into existing confidence intervals that contain the reward function~$f$ with high probability.

\begin{theorem}[Confidence intervals] \label{th:error_bound}
Under the same assumptions as those of Corollary~\ref{co:RKHS_stop},  let~$B_t$ be returned by Algorithm~\ref{alg:PAC} $\forall t\geq 1$ with~$\kappa,\gamma\in(0,1)$.
Moreover, define~$Q_t(a)$ as in~\eqref{eq:Q} with any~$\delta\in(0,1)$ and~$C_t\coloneqq C_{t-1}\cap Q_t$ with~$C_0=\mathbb R$.
Then, with confidence at least~$1-\kappa$, $f(a) \in  C_t(a)$ holds jointly for all~$a \in \domain$ and for all~$t\geq 1$ with probability at least~$(1-\gamma)(1-\delta)$.
\end{theorem}
\begin{proof}(Idea)
We use classic confidence intervals by \cite{abbasi2013online}.
Then, we combine this result with the PAC RKHS norm over-estimation from Corollary~\ref{co:RKHS_stop} by constructing a product probability space.
We provide a detailed proof in Appendix~\ref{proof:error_bound}.
\end{proof}

Finally, we prove safety of the proposed safe BO algorithm with RKHS norm over-estimation.
\begin{theorem}[Safety] \label{th:safety}
Under the same assumptions as those of Theorem~\ref{th:error_bound}, initialize Algorithm~\ref{alg:global_safe_BO} with a safe set~$S_0\neq\emptyset: f(a)\geq h \; \forall a \in S_0$.
Then, with confidence at least~$1-\kappa$,
$
f(a_t) \geq h
$
jointly $\forall t\geq 1$
with probability at least~$(1-\gamma)(1-\delta)$ when running Algorithm~\ref{alg:global_safe_BO}.
\end{theorem}
\begin{proof}(Idea)
The proof is similar to the proof of Theorem~1 by \cite{sui2015safe}.
However, we replace the Lipschitz continuity therein with an RKHS norm induced continuity from Proposition~3.1 by \cite{Fiedler2023Lipschitz} using the (semi)metric~\eqref{eq:semi-metric}.
Then, we combine the confidence intervals from Theorem~\ref{th:error_bound} with the definition of the safe set to prove that all~$a\in S_t$ are safe with high probability.
We provide a detailed proof in Appendix~\ref{proof:safety}.
\end{proof}

\subsection{Locality} \label{sec:locality}
Thus far, we proposed a safe BO algorithm with theoretical guarantees.
At its heart lies the \emph{data-driven computation of the RKHS norm}, which is required to, \eg compute the safe set~\eqref{eq:safe_set}.
The definition of the safe set implies that the algorithm explores in a neighborhood of already collected samples.
Thus, we may not achieve the high sampling density on the entire parameter space that we would, following Figure~\ref{fig:random_RKHS_functions}, desire for a \emph{tight} RKHS norm over-estimation.
However, as we restrict exploration to the safe subset~$S_t$ of the parameter space~$\domain$, estimating the RKHS norm on~$\domain\setminus S_t$ is superfluous.
Actually, it is precisely in unsafe areas where we expect non-smooth behavior and, hence, large RKHS norms.\footnote{%
Let the reward be the distance to an equilibrium and consider the system $x_{k+1}=ax_k$. 
The system is stable at the equilibrium for, \eg $a=0.9999$ but moves away exponentially from the equilibrium for, \eg $a=1.0001$.
Therefore, a small change in parameter~$a$ causes a significant change in the reward and thus a large local RKHS norm.
}
Thus, considering even the \emph{true} global RKHS norm may yield overly conservative exploration, as also reported by \cite{fiedler2024safety}.
Therefore, we use local RKHS norms---inspired by local Lipschitz methods \citep{jordan2020exactly}---to execute safe BO on sub-domains while inheriting the theoretical guarantees derived for Algorithm~\ref{alg:global_safe_BO}.

Algorithm~\ref{alg:local_safe_BO} summarizes the proposed localized safe BO algorithm with the data-driven RKHS norm over-estimation.
We adopt an \emph{adaptive} notion of locality by forming uniform local cubes around each sample~$a\in a_{1:t}$.
Specifically, we define~$N$ local cubes of width~$(1,\ldots,\mathrm N)\cdot \Delta$ around each sample with hyperparameter~$\Delta>0$.
Besides the local cubes, we preserve the global domain~$\domain$ and naturally recover Algorithm~\ref{alg:global_safe_BO} by setting~$N=0$.
We introduce the notation~$\mathcal C_t\coloneqq \{0,\ldots, t\cdot N\}$ as the set of integers labeling the local cubes and the global domain and use the integer~$c\in \mathcal C_t$ to refer to each object.
Figure~\ref{fig:local_cubes} illustrates the structure of the local cubes.
At each iteration~$t$ and local cube~$c\in\mathcal C_t$, we compute the local RKHS norm and determine a candidate parameter with~\eqref{eq:acquisition}.
We choose the parameter for the next experiment as the most uncertain candidate parameter among all cubes.

\begin{figure}[t]
\centering
\begin{tikzpicture}
\pgfmathsetmacro{\dist}{0.6}
\pgfmathsetmacro{\twodist}{1.2}
\pgfmathsetmacro{\threedist}{1.8}
\pgfmathsetmacro{\fourdist}{2.4}

    \draw[thick] (0,0) rectangle (3.75,3.75) node[anchor=north east] {$c_0$};

    \coordinate (A) at (1.1,2.5);
    \node[fill=gray, circle, inner sep=1pt] at (A) {};
    \node[gray, anchor=south] at (A) {\scriptsize{$a_1$}};

    \draw[aaltoBlue, thick] (A) ++(-\dist/2,-\dist/2) rectangle ++(\dist,\dist);
    \draw[aaltoBlue, thick] (A) ++(-\dist,-\dist) rectangle ++(\twodist, \twodist);
    \draw[aaltoBlue, thick] (A) ++(-\threedist/2,-\threedist/2) rectangle ++(\threedist,\threedist);
    \node[aaltoBlue, anchor=south] at ($(A) + (0, \dist/2)$) {\scriptsize{$c_1$}}; 
\node[aaltoBlue, anchor=south] at ($(A) + (0, \dist)$) {\scriptsize{$c_2$}}; 
\node[aaltoBlue, anchor=south] at ($(A) + (0, \threedist/2)$) {\scriptsize{$c_3$}}; 
    
    \coordinate (B) at (2.75, 1.5);
    \node[fill=gray, circle, inner sep=1pt] at (B) {};
    \node[gray,anchor=south] at (B) {\scriptsize{$a_2$}};

    \draw[aaltoBlue, thick] (B) ++(-\dist/2,-\dist/2) rectangle ++(\dist,\dist);
    \draw[aaltoBlue, thick] (B) ++(-\dist,-\dist) rectangle ++(\twodist, \twodist);
    \draw[aaltoBlue, thick] (B) ++(-\threedist/2,-\threedist/2) rectangle ++(\threedist,\threedist);
    \node[aaltoBlue, anchor=south] at ($(B) + (0, \dist/2)$) {\scriptsize{$c_4$}}; 
\node[aaltoBlue, anchor=south] at ($(B) + (0, \dist)$) {\scriptsize{$c_5$}}; 
\node[aaltoBlue, anchor=south] at ($(B) + (0, \threedist/2)$) {\scriptsize{$c_6$}}; 


\end{tikzpicture}
\caption{\emph{Exemplary structure of the local cubes at iteration~$t=2$.}
Each sample $a_1,a_2$ is the center of~$N=3$ local cubes of edge lengths~$\Delta, 2\Delta$, and~$3\Delta$, respectively.
The global domain~$c_0$ is depicted in black while the local cubes~$c_1,\ldots, c_6$ are illustrated in blue.
}
\label{fig:local_cubes}
\end{figure}
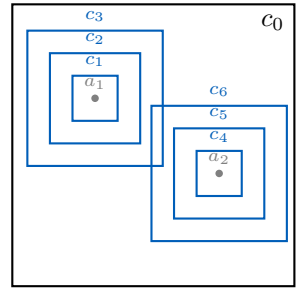

\begin{algorithm}
\begin{algorithmic}[1]   
\Require $k$, $\domain$, $S_0$, $\delta$, $\kappa$, $\gamma$ $m$, $\sigma$, $\Delta$, $N$
\State Init: $a_{1}, y_{1}$ samples corresponding to~$S_0$, $B_0=\infty$
\For{$t=1,2,\ldots$}
\State Compute~$\mathcal C_t$ given~$t$ and~$N$
\For{$c \in \mathcal C_t$} \Comment{Iterate through sub-domains}
\State Determine~$\domain_c\subseteq\domain$, $a_{1:t,c}\subseteq \domain_c$, and~$y_{1:t,c} \subseteq \phantom{forfor} y_{1:t}$ given~$c$ and~$\Delta$
\State $a_{t+1,c}, \,\omega_{t,c}(a_{t+1,c})\gets $ Algorithm~\ref{alg:acquisition}
\EndFor
\State $a_{t+1}\gets \argmax_{a_{t+1,c}, c\in\mathcal C_t} \omega_{t,c}(a_{t+1,c})$
\State  $y_{t+1} \gets  f(a_{t+1})+\epsilon_{t+1}$ \Comment{Conduct experiments} 
\EndFor
\State \Return Best safely evaluable parameter
\end{algorithmic}
\caption{Localized safe BO algorithm with PAC RKHS norm over-estimation}
\label{alg:local_safe_BO}
\end{algorithm}

Besides exploration benefits, the localized approach significantly improves the scalability of discretized BO algorithms like \safeopt.
These discretized BO algorithms suffer from the curse of dimensionality since either the computational and memory complexities grow exponentially or we must accept a coarser discretization; the latter implying exponentially growing distances between the samples, in the worst case causing an empty safe set.
The localized approach sequentially loops through each local cube when acquiring the next sample.
This enables separate discretization in each local cube, which increases the discretization density and, therefore, simplifies exploration.

The following corollary formally states the inherited theoretical guarantees of Algorithm~\ref{alg:local_safe_BO}.

\begin{corollary}[Localized safe BO]
Choose any~$N\in\mathbb{N}$, any~$\Delta>0$,
consider any~$t\geq 1$, and any~$c\in\mathcal C_t$.
Define $f_c{:}\;\domain_c\subseteq\domain\rightarrow\mathbb R$, $f_c(a)=f(a)$ for all~$a\in\domain_c$ and assume that~$f_c\in H_k$, \ie $\|f_c\|_k<\infty$.
Moreover, let all assumptions of Theorems~\ref{th:RKHS_scenario}-\ref{th:safety} and Corollary~\ref{co:RKHS_stop} hold for~$f_c$.
Then, the results from Theorems~\ref{th:RKHS_scenario}-\ref{th:safety} and Corollary~\ref{co:RKHS_stop} directly apply for the local reward functions~$f_c$ and Algorithm~\ref{alg:local_safe_BO}.
\end{corollary}
\begin{proof}
    Instead of deriving the mathematical statements only for the function~$f$ on the global domain~$\domain$, they are derived for~$f_c$ on~$\domain_c$ for all~$c\in\mathcal C_t$ at any iteration~$t\geq 1$.
    Since Algorithm~\ref{alg:local_safe_BO} only samples from the corresponding safe sets, safety directly follows from Theorem~\ref{th:safety}.
\end{proof}

\section{RELATED WORK}\label{sec:related_work}
Next, we relate our safe BO algorithm with RKHS norm over-estimation to the state-of-the-art.

\safeopt\ \citep{sui2015safe} and its extensions \citep{berkenkamp2016safe, sui2018stagewise} require an upper bound on the RKHS norm of the unknown reward function to prove safety with high probability.
The impracticability of this assumption has been addressed by \cite{fiedler2024safety} by proposing an algorithm similar to \safeopt, which instead relies on a priori upper bounds on \emph{(i)} the noise and \emph{(ii)} the Lipschitz constant of the unknown reward function; both of which are unknown and estimating the Lipschitz constant is similarly nontrivial as estimating RKHS norms. 
\cite{berkenkamp2019no} investigate BO with unknown hyperparameters by decreasing the length scale and increasing the RKHS norm to construct an iteratively richer RKHS that will eventually contain the ground truth.
However, they do not provide safety nor guarantee \emph{when} the RKHS contains the function.

Only a few works tackle the RKHS norm estimation.
\cite{Hashimoto2020Nonlinear} and \cite{Scharnhorst2022Robust} observe that the RKHS norm of the approximating function under-estimates the RKHS norm of the ground truth.
Nevertheless, for safety guarantees, we require an over-estimation.
Based on these works, \cite{Tokmak2023Arxiv} propose a simple RKHS norm extrapolation, which empirically results in an upper bound, however, without statistical guarantees and in a noise-free setting with equidistant samples.
Moreover, \cite{karvonen2022error} empirically estimates a computable error bound by using the RKHS norm of the GP mean, similar to the idea presented in Equation~(10) by \cite{fasshauer2011positive}.
Both works do not provide any guarantees.

The only other work we are aware of that develops a safe BO algorithm capable of estimating the RKHS norm with theoretical guarantees is \cite{tokmak2024pacsbo}. 
In contrast, we prove safety guarantees for the resulting safe BO algorithm (Theorem~\ref{th:safety}) instead of only providing statistical guarantees for the RKHS norm over-estimation.
Moreover, we obtain a tighter RKHS norm over-estimation by using a sampling-and-discarding scenario approach instead of Hoeffding's inequality.
We compare the tightness in Appendix~\ref{app:tightness}.
Further, our algorithm is \emph{significantly} more scalable by using an \emph{adaptive} notion of locality, allowing for separate discretization in each sub-domain.
We elaborate on the improved scalability through adaptive locality in Appendix~\ref{app:adaptive_locality}.
Finally, we work under less restricting and more interpretable assumptions as discussed in Appendix~\ref{app:assumption_rebuttal}.

\section{EXPERIMENTS}\label{sec:experiments}
In this section, we first numerically investigate the scenario approach.
Then, we evaluate Algorithm~\ref{alg:local_safe_BO}  and compare it with \safeopt.
Specifically, we illustrate the impact of estimating the RKHS norm instead of randomly guessing it in a one-dimensional toy experiment before comparing both algorithms on challenging RL benchmarks.
Finally, we demonstrate the practicability of our algorithm by optimizing a controller for a real Furuta pendulum \citep{furuta1992swing}.
All experiments were conducted with hyperparameters~$\sigma=10^{-2}, \delta=10^{-2}, \gamma=10^{-1},\kappa=10^{-2}, \overline\alpha=1$, $m=1000$, and~$\hat N_c=\max\{500 \mathrm{width}(\domain_c), t+10\}$.
Moreover, we shift and normalize the domains to yield~$\domain=[0,1]^n$ and use the Matérn32 kernel with~$\ell=0.1$ unless stated otherwise.

\paragraph{Scenario approach}
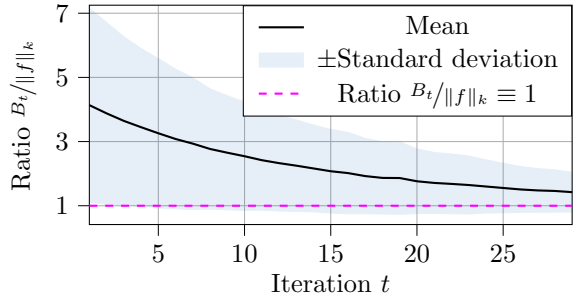
\begin{figure}
\centering
\begin{tikzpicture}

\definecolor{darkgray176}{RGB}{176,176,176}
\definecolor{lightblue}{RGB}{173,216,230}
\definecolor{magenta}{RGB}{255,0,255}
\pgfplotsset{
every axis legend/.append style={
at={(1, 1)},
anchor=north east,
},
}
\begin{axis}[
tick align=outside,
tick pos=left,
x grid style={darkgray176},
xlabel={Iteration~$t$},
xmin=1, xmax=29,
xtick style={color=black},
y grid style={darkgray176},
ylabel={Ratio $\nicefrac{B_t}{\|f\|_k}$},
width=8cm,
ytick={1,3,5,7},
    ylabel style={yshift=-1em}, 
height=4.5cm,
ymin=0.413651417952565, ymax=7.25,
ytick style={color=black},
grid=both,
]
\path [draw=aaltoBlue, fill=aaltoBlue, opacity=0.1]
(axis cs:1,7.16042221773556)
--(axis cs:1,1.11009610363831)
--(axis cs:2,1.04732360078726)
--(axis cs:3,1.02587165505926)
--(axis cs:4,0.994059129846926)
--(axis cs:5,0.932517168892915)
--(axis cs:6,0.904504013881886)
--(axis cs:7,0.883046131843942)
--(axis cs:8,0.894937838484696)
--(axis cs:9,0.875699410160979)
--(axis cs:10,0.853498080819181)
--(axis cs:11,0.863147594126054)
--(axis cs:12,0.836740533911606)
--(axis cs:13,0.822699976747126)
--(axis cs:14,0.811496583139754)
--(axis cs:15,0.773307223712149)
--(axis cs:16,0.748951765475134)
--(axis cs:17,0.749104874239569)
--(axis cs:18,0.741069823198915)
--(axis cs:19,0.734926217942231)
--(axis cs:20,0.756894114747096)
--(axis cs:21,0.761898568595918)
--(axis cs:22,0.750094977200294)
--(axis cs:23,0.741621440864143)
--(axis cs:24,0.772522406499804)
--(axis cs:25,0.781942034559374)
--(axis cs:26,0.794217728755342)
--(axis cs:27,0.79628322159567)
--(axis cs:28,0.801476569056672)
--(axis cs:29,0.806597437972088)
--(axis cs:29,2.03996196341065)
--(axis cs:29,2.03996196341065)
--(axis cs:28,2.11973907297098)
--(axis cs:27,2.15792859631459)
--(axis cs:26,2.2215380525058)
--(axis cs:25,2.32472189673878)
--(axis cs:24,2.428861488062)
--(axis cs:23,2.54970633150165)
--(axis cs:22,2.61083346062668)
--(axis cs:21,2.65971878985793)
--(axis cs:20,2.76781176781792)
--(axis cs:19,2.99165646486777)
--(axis cs:18,2.99344417126849)
--(axis cs:17,3.09900117081778)
--(axis cs:16,3.2863403561777)
--(axis cs:15,3.37686349440062)
--(axis cs:14,3.51268914102317)
--(axis cs:13,3.67535672630223)
--(axis cs:12,3.8187688269661)
--(axis cs:11,3.97733220363973)
--(axis cs:10,4.23229297661537)
--(axis cs:9,4.42951837690876)
--(axis cs:8,4.64375168891368)
--(axis cs:7,4.98797824195428)
--(axis cs:6,5.25603562071026)
--(axis cs:5,5.58495267743436)
--(axis cs:4,5.90137138916865)
--(axis cs:3,6.26267280883328)
--(axis cs:2,6.70994751599485)
--(axis cs:1,7.16042221773556)
--cycle;


\addplot [thick, black]
table {%
1 4.13525915145874
2 3.87863564491272
3 3.6442723274231
4 3.44771528244019
5 3.25873494148254
6 3.0802698135376
7 2.93551230430603
8 2.76934480667114
9 2.65260887145996
10 2.54289555549622
11 2.42023992538452
12 2.32775473594666
13 2.24902844429016
15 2.07508540153503
16 2.01764607429504
17 1.92405307292938
18 1.86725699901581
19 1.86329138278961
20 1.76235294342041
21 1.710808634758
22 1.68046426773071
23 1.64566385746002
25 1.55333197116852
26 1.50787794589996
27 1.47710585594177
28 1.4606077671051
29 1.42327964305878
};

\addplot [line width=0.2cm, draw=aaltoBlue, fill=aaltoBlue, opacity=0.1]
table {%
-2 4.13525915145874
-1 3.87863564491272
};

\addplot [thick, magenta, dashed]
table {%
1 1
29 1
};
\legend{Mean, $\pm $Standard deviation, Ratio $\nicefrac{B_t}{\|f\|_k}\equiv 1$}
\end{axis}

\end{tikzpicture}
                \caption{\emph{Numerical investigation of the RKHS norm over-estimation.} For an increasing sample size, we receive tighter bounds.
    }
            \label{fig:numerical_investigation}
\end{figure}
To test Corollary~\ref{co:RKHS_stop}, we create 200 RKHS functions with RKHS norms sampled uniformly from~$[1,10]$.
We sample the number of center points for each RKHS function uniformly from~$[100,1000]$, and scale the corresponding coefficients~$\alpha$ to satisfy the pre-determined~$\|f\|_k$. 
At each iteration, we compute the over-estimations~$B_t$ using Algorithm~\ref{alg:PAC} for each RKHS function~$f$ and append a new parameter sampled uniformly from~$\domain$.
We present the numerical investigation in Figure~\ref{fig:numerical_investigation}.
As already discussed in Section~\ref{sec:RKHS}, we see that the RKHS norm over-estimation gets tighter for an increasing sample set, supporting the sensibility of the proposed RKHS norm over-estimation.
Crucially, in only two out of 200 cases did Algorithm~\ref{alg:PAC} under-estimate the RKHS norm.
As we chose~$\gamma=10^{-1}$ and $\kappa=10^{-2}$, this is well within the guaranteed range specified in Corollary~\ref{co:RKHS_stop}.
Moreover, we investigate the required computation time for the scenario approach in Appendix~\ref{app:scenario_time} and conduct an ablation study in Appendix~\ref{app:ablation_scenario}.


\paragraph{Numerical experiments}
To illustrate the benefits of our algorithm compared to \safeopt, we let both maximize a synthetic function~$f\in H_k$ generated with 1000 random center points~$x$ and coefficients~$\alpha$ scaled to yield~$\|f\|_k=5$, which we present in Figure~\ref{fig:1D_toy}.
For \safeopt, we perform two runs, one with an over-estimation ($B_t\equiv 25$, center) and one with an under-estimation ($B_t\equiv 1$, right) of the RKHS norm. 
The former yields conservative exploration (crucially, it does not find the optimum within the given number of iterations), while the latter incurs failures (red crosses).
In contrast, our algorithm (left) stays safe and finds the optimum.
For Algorithm~\ref{alg:local_safe_BO}, we used~$N=5$ and~$\Delta = 0.1$.
We provide ablation studies with different kernels and different locality parameters in Appendices~\ref{app:numerical_ablation} and~\ref{app:locality_ablation}, respectively.


\begin{figure*}[t]
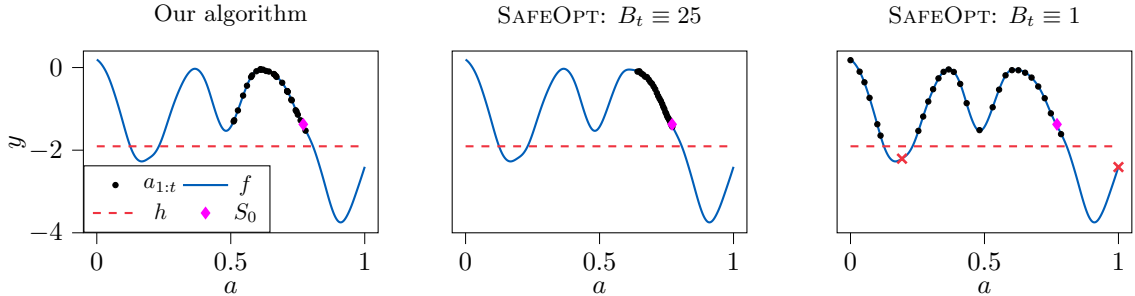

    \centering   
\begin{subfigure}{5.5cm}
        \input{final_figs/1D_toy/PACSBO.tex}
\end{subfigure}
\begin{subfigure}{5cm}
        \input{final_figs/1D_toy/SafeOpt_over-estimation}
\end{subfigure}
\begin{subfigure}{5cm}
        \input{final_figs/1D_toy/SafeOpt_under-estimation}
\end{subfigure}
        \caption{
\emph{Toy example to compare Algorithm~\ref{alg:local_safe_BO} with \safeopt.}
Algorithm~\ref{alg:local_safe_BO} (left) explores the domain and stays safe, while \safeopt\ is either too conservative (center) or samples unsafely (right).
}
     \label{fig:1D_toy}
\end{figure*}

\paragraph{RL benchmarks}
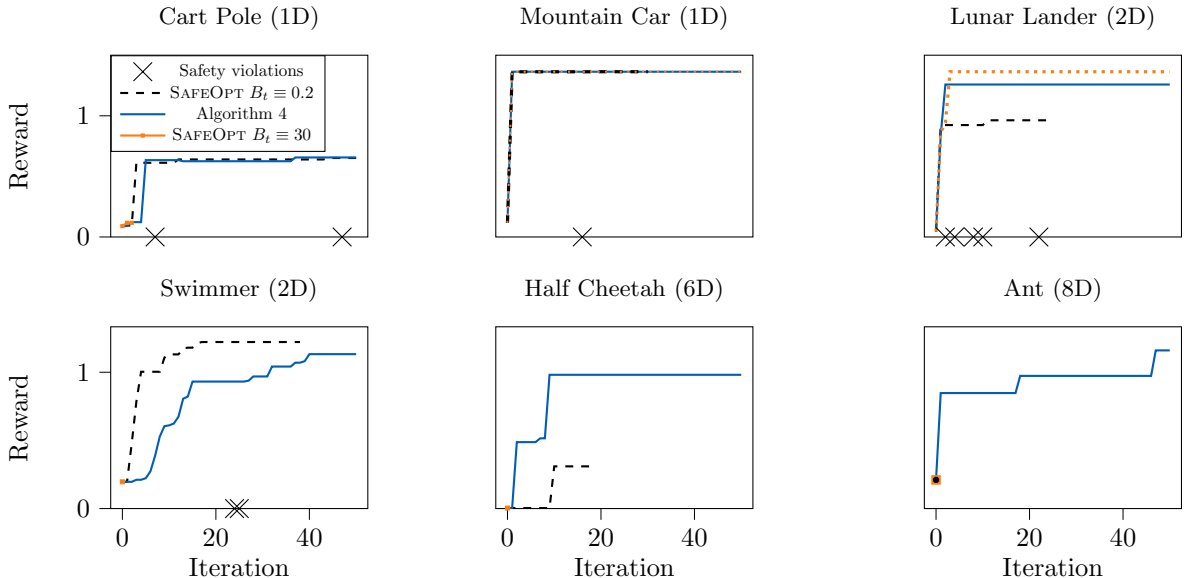
\begin{figure*}
    \begin{subfigure}{5.75cm}
\begin{tikzpicture}

\definecolor{darkgray176}{RGB}{176,176,176}
\definecolor{darkorange25512714}{RGB}{255,127,14}
\definecolor{forestgreen4416044}{RGB}{44,160,44}
\definecolor{lightgray204}{RGB}{204,204,204}
\definecolor{steelblue31119180}{RGB}{31,119,180}
\pgfplotsset{
every axis legend/.append style={
at={(0, 1)},
anchor=north west,
},
}

\begin{axis}[
tick align=outside,
tick pos=left,
x grid style={darkgray176},
xmin=-2.5, xmax=52.5,
xtick style={color=black},
y grid style={darkgray176},
ymin=0.5, ymax=0.8,
height=4cm,
width=5cm,
title={\small{Cart Pole (1D)}},
ytick style={color=black},
ylabel={Reward},
xtick style={draw=none},
legend style={nodes={scale=0.65}, legend columns=1}, 
xticklabels={,,},
ytick={0.5,0.7},
yticklabels={0, 1},
]
\addplot [draw=black, fill=black, mark=x, only marks, mark size=5]
table {%
7 0.5
47 0.5
};
\addplot [thick, black, dashed]
table {%
0 0.51800549030304
1 0.518889784812927
2 0.518889784812927
3 0.622481942176819
4 0.622481942176819
5 0.622481942176819
6 0.622481942176819
7 0.622481942176819
8 0.622481942176819
9 0.622481942176819
10 0.622481942176819
11 0.622481942176819
12 0.627945959568024
13 0.627945959568024
14 0.627945959568024
15 0.627945959568024
16 0.627945959568024
17 0.627945959568024
18 0.627945959568024
19 0.627945959568024
20 0.627945959568024
21 0.627945959568024
22 0.627945959568024
23 0.627945959568024
24 0.627945959568024
25 0.627945959568024
26 0.627945959568024
27 0.627945959568024
28 0.627945959568024
29 0.627945959568024
30 0.627945959568024
31 0.627945959568024
32 0.627945959568024
33 0.627945959568024
34 0.627945959568024
35 0.627945959568024
36 0.627945959568024
37 0.627945959568024
38 0.627945959568024
39 0.627945959568024
40 0.627945959568024
41 0.627945959568024
42 0.627945959568024
43 0.627945959568024
44 0.630361199378967
45 0.630361199378967
46 0.630361199378967
47 0.630361199378967
48 0.630361199378967
49 0.630361199378967
50 0.630361199378967
};

\addplot [thick, aaltoBlue]
table {%
0 0.51800549030304
1 0.51800549030304
2 0.524470567703247
3 0.524470567703247
4 0.524470567703247
5 0.626824259757996
6 0.626824259757996
7 0.626824259757996
8 0.626824259757996
9 0.626824259757996
10 0.626824259757996
11 0.626824259757996
12 0.626824259757996
13 0.625010430812836
14 0.625010430812836
15 0.625010430812836
16 0.625010430812836
17 0.625010430812836
18 0.625010430812836
19 0.625010430812836
20 0.625010430812836
21 0.625010430812836
22 0.625010430812836
23 0.625010430812836
24 0.625010430812836
25 0.625010430812836
26 0.625010430812836
27 0.625010430812836
28 0.625010430812836
29 0.625010430812836
30 0.625010430812836
31 0.625010430812836
32 0.625010430812836
33 0.625010430812836
34 0.625010430812836
35 0.625010430812836
36 0.625010430812836
37 0.631279170513153
38 0.631279170513153
39 0.631279170513153
40 0.631279170513153
41 0.631279170513153
42 0.631279170513153
43 0.631279170513153
44 0.631279170513153
45 0.631279170513153
46 0.631279170513153
47 0.631279170513153
48 0.631279170513153
49 0.631279170513153
50 0.631279170513153
};

\addplot [thick, darkorange25512714, mark=square*, mark size=0.4]
table {%
0 0.51800549030304
1 0.523368537425995
2 0.523368537425995
};
\addplot [semithick, red, dashed]
table {%
0 0
1 0
2 0
3 0
4 0
5 0
6 0
7 0
8 0
9 0
10 0
11 0
12 0
13 0
14 0
15 0
16 0
17 0
18 0
19 0
20 0
21 0
22 0
23 0
24 0
25 0
26 0
27 0
28 0
29 0
30 0
31 0
32 0
33 0
34 0
35 0
36 0
37 0
38 0
39 0
40 0
41 0
42 0
43 0
44 0
45 0
46 0
47 0
48 0
49 0
50 0
};
\legend{Safety violations, \safeopt\ $B_t\equiv 0.2$,Algorithm~4, \safeopt\ $B_t\equiv 30$}
\end{axis}
\end{tikzpicture}
    \end{subfigure}
        \begin{subfigure}{5.25cm}
\begin{tikzpicture}

\definecolor{darkgray176}{RGB}{176,176,176}
\definecolor{darkorange25512714}{RGB}{255,127,14}
\definecolor{forestgreen4416044}{RGB}{44,160,44}
\definecolor{lightgray204}{RGB}{204,204,204}
\definecolor{steelblue31119180}{RGB}{31,119,180}

\begin{axis}[
legend cell align={left},
legend style={
  fill opacity=0.8,
  draw opacity=1,
  text opacity=1,
  at={(0.5,0.5)},
  anchor=center,
  draw=lightgray204
},
tick align=outside,
tick pos=left,
x grid style={darkgray176},
xmin=-2.5, xmax=52.5,
y grid style={darkgray176},
ymin = 0.5,
height=4cm,
width=5cm,
title={\small{Mountain Car (1D)}},
xtick style={draw=none},
yticklabels={,,},
xticklabels={,,},
ytick style={draw=none}
]
\addplot [thick, aaltoBlue]
table {%
0 0.555805385112762
1 1.15865695476532
50 1.15865695476532
};
\addplot [draw=black, fill=black, mark=x, only marks, mark size=5]
table {%
16 0.5
};
\addplot [very thick, black, dashed]
table {%
0 0.555805385112762
1 1.15841126441956
19 1.15841126441956
20 1.15865516662598
30 1.15865516662598
};


\addplot [thick, darkorange25512714, dotted]
table {%
0 0.555805385112762
1 1.15867257118225
50 1.15867257118225
};
\addplot [semithick, red, dashed]
table {%
0 0
1 0
2 0
3 0
4 0
5 0
6 0
7 0
8 0
9 0
10 0
11 0
12 0
13 0
14 0
15 0
16 0
17 0
18 0
19 0
20 0
21 0
22 0
23 0
24 0
25 0
26 0
27 0
28 0
29 0
30 0
31 0
32 0
33 0
34 0
35 0
36 0
37 0
38 0
39 0
40 0
41 0
42 0
43 0
44 0
45 0
46 0
47 0
48 0
49 0
50 0
};
\end{axis}

\end{tikzpicture}
    \end{subfigure}
        \begin{subfigure}{5.25cm}
\begin{tikzpicture}

\definecolor{darkgray176}{RGB}{176,176,176}
\definecolor{darkorange25512714}{RGB}{255,127,14}
\definecolor{forestgreen4416044}{RGB}{44,160,44}
\definecolor{lightgray204}{RGB}{204,204,204}
\definecolor{steelblue31119180}{RGB}{31,119,180}

\begin{axis}[
legend cell align={left},
legend style={
  fill opacity=0.8,
  draw opacity=1,
  text opacity=1,
  at={(0.5,0.5)},
  anchor=center,
  draw=lightgray204
},
tick align=outside,
tick pos=left,
x grid style={darkgray176},
xmin=-2.5, xmax=52.5,
xtick style={color=black},
y grid style={darkgray176},
ymin=0.8,
height=4cm,
width=5cm,
title={\small{Lunar Lander (2D)}},
ytick style={color=black},
xtick style={draw=none},
ytick style={draw=none},
yticklabels={,,},
xticklabels={,,},
]
\addplot [draw=black, fill=black, mark=x, only marks, mark size=5]
table {%
2 0.8
4 0.8
8 0.8
10 0.8
22 0.8
};
\addplot [thick, black, dashed]
table {%
0 0.830873966217041
1 1.44357216358185
10 1.44357216358185
11 1.47105014324188
24 1.47105014324188
};

\addplot [thick, aaltoBlue]
table {%
0 0.830873966217041
1 1.38639044761658
2 1.67674517631531
50 1.67674517631531
};

\addplot [very thick, darkorange25512714, dotted]
table {%
0 0.830873966217041
1 1.4008367061615
2 1.45636069774628
3 1.75027096271515
50 1.75027096271515
};
\addplot [semithick, red, dashed]
table {%
0 0
1 0
2 0
3 0
4 0
5 0
6 0
7 0
8 0
9 0
10 0
11 0
12 0
13 0
14 0
15 0
16 0
17 0
18 0
19 0
20 0
21 0
22 0
23 0
24 0
25 0
26 0
27 0
28 0
29 0
30 0
31 0
32 0
33 0
34 0
35 0
36 0
37 0
38 0
39 0
40 0
41 0
42 0
43 0
44 0
45 0
46 0
47 0
48 0
49 0
50 0
};
\end{axis}

\end{tikzpicture}
    \end{subfigure}
        \begin{subfigure}{5.75cm}
\begin{tikzpicture}

\definecolor{darkgray176}{RGB}{176,176,176}
\definecolor{darkorange25512714}{RGB}{255,127,14}
\definecolor{forestgreen4416044}{RGB}{44,160,44}
\definecolor{lightgray204}{RGB}{204,204,204}
\definecolor{steelblue31119180}{RGB}{31,119,180}

\begin{axis}[
legend cell align={left},
legend style={
  fill opacity=0.8,
  draw opacity=1,
  text opacity=1,
  at={(0.5,0.5)},
  anchor=center,
  draw=lightgray204
},
tick align=outside,
tick pos=left,
x grid style={darkgray176},
xmin=-2.5, xmax=52.5,
xtick style={color=black},
y grid style={darkgray176},
ymin=0, ymax=0.4,
height=4cm,
width=5cm,
ylabel={Reward},
title={\small{Swimmer (2D)}},
ytick style={color=black},
ytick={0,0.3},
yticklabels={0, 1},
xlabel={Iteration}
]
\addplot [draw=black, fill=black, mark=x, only marks, mark size=5]
table {%
24 0
25 0
};
\addplot  [thick, black, dashed]
table {%
0 0.0588392242789268
1 0.0588392242789268
2 0.143409386277199
3 0.235898897051811
4 0.301049441099167
8 0.301049441099167
9 0.332220822572708
10 0.339309751987457
12 0.339309751987457
13 0.352369576692581
14 0.354056477546692
15 0.354088634252548
16 0.363837331533432
17 0.366511911153793
38 0.366511911153793
};
\addplot [thick, aaltoBlue]
table {%
0 0.0588392242789268
2 0.0588392242789268
3 0.0633627325296402
4 0.0633627325296402
5 0.0668168142437935
6 0.0819131433963776
7 0.116209730505943
8 0.158422812819481
9 0.181379839777946
10 0.183094352483749
11 0.187168657779694
12 0.202136784791946
13 0.241676345467567
14 0.246660575270653
15 0.279466569423676
26 0.279466569423676
27 0.281439393758774
28 0.290706872940063
31 0.290706872940063
32 0.31252259016037
36 0.31252259016037
37 0.321031391620636
38 0.321031391620636
39 0.324390232563019
40 0.339804083108902
50 0.339804083108902
};
\addplot [draw=darkorange25512714, fill=darkorange25512714, mark=square*, only marks, mark size=0.75]
table {%
0 0.0588392242789268
};
\end{axis}

\end{tikzpicture}
    \end{subfigure}
    \hspace*{0.225cm}
        \begin{subfigure}{5.25cm}
\begin{tikzpicture}

\definecolor{darkgray176}{RGB}{176,176,176}
\definecolor{steelblue31119180}{RGB}{31,119,180}
\definecolor{darkorange25512714}{RGB}{255,127,14}
\begin{axis}[
tick align=outside,
tick pos=left,
x grid style={darkgray176},
xmin=-2.5, xmax=52.5,
xtick style={color=black},
width=5cm,
height=4cm,
title={\small{Half Cheetah (6D)}},
y grid style={darkgray176},
ymin=4.59, ymax=5.79306499958038,
ytick style={color=black},
xlabel={Iteration},
yticklabels={,,},
ytick style={draw=none}
]
\addplot [thick, black, dashed]
table {%
0 4.59316349029541
1 4.59316349029541
2 4.59316349029541
3 4.59316349029541
4 4.59316349029541
5 4.59316349029541
6 4.59316349029541
7 4.59316349029541
8 4.59316349029541
9 4.59316349029541
10 4.86924600601196
11 4.86924600601196
12 4.86924600601196
13 4.86924600601196
14 4.86924600601196
15 4.86924600601196
16 4.86924600601196
17 4.86924600601196
18 4.86924600601196
};
\addplot [draw=darkorange25512714, fill=darkorange25512714, mark=square*, only marks, mark size=0.75]
table {%
0 4.59316349029541
};
\addplot [thick, aaltoBlue]
table {%
0 4.59316349029541
1 4.59381103515625
2 5.0296049118042
3 5.0296049118042
4 5.0296049118042
5 5.0296049118042
6 5.0296049118042
7 5.05480051040649
8 5.05480051040649
9 5.47524642944336
10 5.47524642944336
11 5.47524642944336
12 5.47524642944336
13 5.47524642944336
14 5.47524642944336
15 5.47524642944336
16 5.47524642944336
17 5.47524642944336
18 5.47524642944336
19 5.47524642944336
20 5.47524642944336
21 5.47524642944336
22 5.47524642944336
23 5.47524642944336
24 5.47524642944336
25 5.47524642944336
26 5.47524642944336
27 5.47524642944336
28 5.47524642944336
29 5.47524642944336
30 5.47524642944336
31 5.47524642944336
32 5.47524642944336
33 5.47524642944336
34 5.47524642944336
35 5.47524642944336
36 5.47524642944336
37 5.47524642944336
38 5.47524642944336
39 5.47524642944336
40 5.47524642944336
41 5.47524642944336
42 5.47524642944336
43 5.47524642944336
44 5.47524642944336
45 5.47524642944336
46 5.47524642944336
47 5.47524642944336
48 5.47524642944336
49 5.47524642944336
50 5.47524642944336
};
\addplot [semithick, aaltoRed, dashed]
table {%
0 0
1 0
2 0
3 0
4 0
5 0
6 0
7 0
8 0
9 0
10 0
11 0
12 0
13 0
14 0
15 0
16 0
17 0
18 0
19 0
20 0
21 0
22 0
23 0
24 0
25 0
26 0
27 0
28 0
29 0
30 0
31 0
32 0
33 0
34 0
35 0
36 0
37 0
38 0
39 0
40 0
41 0
42 0
43 0
44 0
45 0
46 0
47 0
48 0
49 0
50 0
};
\end{axis}
\end{tikzpicture}
    \end{subfigure}
    \hspace*{0.225cm}
        \begin{subfigure}{5.25cm}
    \begin{tikzpicture}

\definecolor{darkgray176}{RGB}{176,176,176}
\definecolor{darkorange25512714}{RGB}{255,127,14}
\definecolor{forestgreen4416044}{RGB}{44,160,44}
\definecolor{lightgray204}{RGB}{204,204,204}
\definecolor{steelblue31119180}{RGB}{31,119,180}
\pgfplotsset{
every axis legend/.append style={
at={(1, 0)},
anchor=south east,
},
}
\begin{axis}[
tick align=outside,
tick pos=left,
x grid style={darkgray176},
xmin=-2.5, xmax=52.5,
xtick style={color=black},
y grid style={darkgray176},
ymin=0.6, ymax=1.3,
width=5cm,
height=4cm,
title={\small{Ant (8D)}},
ytick style={color=black},
xlabel={Iteration},
ytick style={draw=none},
yticklabels={,,},
]

\addplot [thick, aaltoBlue]
table {%
0 0.710347235202789
1 1.04483699798584
2 1.04483699798584
3 1.04483699798584
4 1.04483699798584
5 1.04483699798584
6 1.04483699798584
7 1.04483699798584
8 1.04483699798584
9 1.04483699798584
10 1.04483699798584
11 1.04483699798584
12 1.04483699798584
13 1.04483699798584
14 1.04483699798584
15 1.04483699798584
16 1.04483699798584
17 1.04483699798584
18 1.11092531681061
19 1.11092531681061
20 1.11092531681061
21 1.11092531681061
22 1.11092531681061
23 1.11092531681061
24 1.11092531681061
25 1.11092531681061
26 1.11092531681061
27 1.11092531681061
28 1.11092531681061
29 1.11092531681061
30 1.11092531681061
31 1.11092531681061
32 1.11092531681061
33 1.11092531681061
34 1.11092531681061
35 1.11092531681061
36 1.11092531681061
37 1.11092531681061
38 1.11092531681061
39 1.11092531681061
40 1.11092531681061
41 1.11092531681061
42 1.11092531681061
43 1.11092531681061
44 1.11092531681061
45 1.11092531681061
46 1.11092531681061
47 1.20910227298737
48 1.20910227298737
49 1.20910227298737
50 1.20910227298737
};

\addplot [semithick, darkorange25512714, mark=square*, mark size=1.5]  
table {%
0 0.710347235202789
};

\addplot [thick, black, mark=*, mark size=0.8]
table {%
0 0.710347235202789
};

\end{axis}

\end{tikzpicture}
    \end{subfigure}
        \caption{
        \emph{RL benchmarks.}
        We optimize SAC policies by learning an additive bias in a sim-to-real inspired setting.
        Algorithm~\ref{alg:local_safe_BO} exhibits better scalability, safety, and performance than \safeopt.
        We plot the maximum scaled reward encountered over iterations and mark violations of $h=0$ with crosses.}
        \label{fig:gym}
\end{figure*}
Next, we evaluate our algorithm and compare it to \safeopt\ in challenging simulation benchmarks.
In particular, we consider a sim-to-real setting, where no safety guarantees are required during simulation.
Thus, we train policies in simulation using the soft actor-critic (SAC) algorithm \citep{haarnoja2018soft, stable-baselines3}.
Those RL policies map from the states to the actions in~$\mathbb R^n$ for the cart pole ($n=1$), mountain car ($n=1$), swimmer ($n=2$), lunar lander ($n=2$), half cheetah ($n=6$), and ant ($n=8$) environments \citep{brockman2016openai,todorov2012mujoco}.
Then, to imitate real-world experiments, we manipulate the environments by, \eg adding a wind disturbance for the lunar lander; see Appendix~\ref{app:RL} for details.
Thus, the policies learned with SAC still provide a safe starting point but are not optimal anymore.
As we now must guarantee safety, we optimize these initial policies by learning an additive bias term~$b\in\mathbb R^n$ using Algorithm~\ref{alg:local_safe_BO} and \safeopt.
Figure~\ref{fig:gym} shows the rewards over iterations for the different environments. 
Algorithm~\ref{alg:local_safe_BO} stays safe and learns a bias that improves the reward for all environments.
For \safeopt, a small RKHS norm leads to frequent safety violations (black crosses), which, \eg correspond to the lunar lander crashing, whereas a large RKHS norm mostly yields conservative exploration or premature stopping.
Importantly, even \safeopt\ with a small RKHS norm fails to explore noticeably in the half cheetah and ant environments, which is due to the coarse discretization in high dimensions, whereas our method improves scalability by exploiting locality and successfully improves the reward.

\paragraph{Hardware experiment}
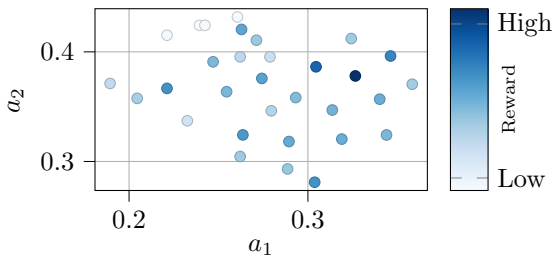
\begin{figure}
\centering
\begin{tikzpicture}

\definecolor{darkgray176}{RGB}{176,176,176}
\definecolor{magenta}{RGB}{255,0,255}
\pgfplotsset{
every axis legend/.append style={
at={(0, 0.2)},
anchor=south west,
},
}
\begin{axis}[
colorbar,
colorbar style={ytick={50,440},yticklabels={Low,High},ylabel={\scriptsize{Reward}},
ylabel style={rotate=0, yshift=-2cm},
},
colormap={mymap}{[1pt]
  rgb(0pt)=(0.968627450980392,0.984313725490196,1);
  rgb(1pt)=(0.870588235294118,0.92156862745098,0.968627450980392);
  rgb(2pt)=(0.776470588235294,0.858823529411765,0.937254901960784);
  rgb(3pt)=(0.619607843137255,0.792156862745098,0.882352941176471);
  rgb(4pt)=(0.419607843137255,0.682352941176471,0.83921568627451);
  rgb(5pt)=(0.258823529411765,0.572549019607843,0.776470588235294);
  rgb(6pt)=(0.129411764705882,0.443137254901961,0.709803921568627);
  rgb(7pt)=(0.0313725490196078,0.317647058823529,0.611764705882353);
  rgb(8pt)=(0.0313725490196078,0.188235294117647,0.419607843137255)
},
point meta max=478.656677246094,
point meta min=18.2441463470459,
tick align=outside,
tick pos=left,
x grid style={darkgray176},
xlabel={$a_1$},
xmin=0.180953337997198, xmax=0.366779953986406,
xtick style={color=black},
y grid style={darkgray176},
ylabel={$a_2$},
height=4cm,
width=6cm,
    ylabel style={yshift=-0.5em}, 
ymin=0.273522758483887, ymax=0.439356064796448,
ytick style={color=black},
ytick={0.3, 0.4},
xtick={0.2, 0.3},
grid=both
]

\addplot [
  colormap={mymap}{[1pt]
  rgb(0pt)=(0.968627450980392,0.984313725490196,1);
  rgb(1pt)=(0.870588235294118,0.92156862745098,0.968627450980392);
  rgb(2pt)=(0.776470588235294,0.858823529411765,0.937254901960784);
  rgb(3pt)=(0.619607843137255,0.792156862745098,0.882352941176471);
  rgb(4pt)=(0.419607843137255,0.682352941176471,0.83921568627451);
  rgb(5pt)=(0.258823529411765,0.572549019607843,0.776470588235294);
  rgb(6pt)=(0.129411764705882,0.443137254901961,0.709803921568627);
  rgb(7pt)=(0.0313725490196078,0.317647058823529,0.611764705882353);
  rgb(8pt)=(0.0313725490196078,0.188235294117647,0.419607843137255)
},
  only marks,
  scatter,
  scatter src=explicit
]
table [x=x, y=y, meta=colordata]{%
x  y  colordata
0.239399999380112 0.424199998378754 18.2441463470459
0.242424249649048 0.424242436885834 24.4306697845459
0.262878775596619 0.420454561710358 266.0604248046875
0.271212130784988 0.41060608625412 194.66806030273438
0.260606050491333 0.431818187236786 21.762081146240234
0.262121230363846 0.395454555749893 146.57666015625
0.278787881135941 0.395454525947571 118.1636962890625
0.246969684958458 0.390909075737 228.72743225097656
0.274242401123047 0.375757575035095 274.269775390625
0.254545480012894 0.363636374473572 232.5571746826172
0.279545485973358 0.346212118864059 159.85714721679688
0.293187886476517 0.358290910720825 221.17501831054688
0.221212148666382 0.415151536464691 26.14986228942871
0.221212103962898 0.366666704416275 305.3229675292969
0.23257577419281 0.337121248245239 108.15962982177734
0.204545438289642 0.35757577419281 179.0894317626953
0.189400002360344 0.371169686317444 139.3070068359375
0.3045454621315 0.386363655328751 384.4695739746094
0.326515167951584 0.378030270338058 478.65667724609375
0.313642412424088 0.346927285194397 238.38584899902344
0.34015154838562 0.356818228960037 250.29568481445312
0.263636380434036 0.324242442846298 293.6642761230469
0.289393931627274 0.318181782960892 256.39691162109375
0.324242442846298 0.41212123632431 186.3662567138672
0.346212148666382 0.396212100982666 326.84228515625
0.262121230363846 0.3045454621315 189.4579620361328
0.35833328962326 0.370454549789429 185.73330688476562
0.318945467472076 0.32041209936142 252.0086212158203
0.288636416196823 0.293181836605072 196.5853729248047
0.343945473432541 0.324200004339218 230.79847717285156
0.303787887096405 0.281060636043549 306.7078857421875
};
\end{axis}

\end{tikzpicture}
        \caption{\emph{Explored domain and rewards for the hardware experiment.}
    Algorithm~\ref{alg:local_safe_BO} safely optimizes the controller for a Furuta pendulum.
    }
            \label{fig:furuta}
\end{figure}

Lastly, we demonstrate the applicability of Algorithm~\ref{alg:local_safe_BO} to real-world systems by optimizing the balancing controller of a
Furuta pendulum~\cite{furuta1992swing}; see Appendix~\ref{app:hardware} for a visualization of the setup.
We consider a similar experimental setup as \cite{Baumann2021GO}, where the reward function corresponds to the control performance, and we tune the first two entries of a state-feedback controller.
We execute Algorithm~\ref{alg:local_safe_BO} with~$\ell=0.2$, $N=3, \Delta=0.15$ and we have
$S_0=[0.239, 0.424]^\top$.
After 30 iterations, we explored the domain to significantly improve the controller performance while only conducting safe experiments, as shown in the video and in Figure~\ref{fig:furuta}. 
This demonstrates that our algorithm is applicable to safety-critical real-world systems.

\section{LIMITATIONS}\label{sec:limitations}
In this section, we discuss the limitations of our contributions, specifically Assumption~\ref{asm:random}.
Assumption~\ref{asm:random} essentially states that~$f$ and~$\rho_{t,j}$ are i.i.d.\ samples from the same---potentially unknown---probability space.
However, in the frequentist setting,~$f$ is a fixed sample generated by \emph{nature's probability space}. 
Hence, even if we would be able to sample from the entire RKHS~$H_k$ (see Remark~\ref{re:limit}), in practice, we \emph{never} have access to 
nature's probability space and, thus, Assumption~\ref{asm:random} implies a sampling oracle. 
Hence, by generating random RKHS functions, we approximate that probability space and impose a \emph{prior} on~$f$.
Therefore, we essentially are in the Bayesian setting.
However, mixing both frequentist and Bayesian methods is fairly common \citep{bayarri2004interplay, baggio2022bayesian}.
Moreover, assuming an a priori tight upper bound on~$\|f\|_k$ \citep{sui2015safe} or assuming that the expected value of the RKHS norms of the random RKHS functions over-estimates~$\|f\|_k $ \citep{tokmak2024pacsbo} restricts nature's function space and also imposes prior knowledge on~$f$.
Also, we explain the mathematical meaning of Assumption~\ref{asm:random} in Appendix~\ref{app:assumption_rebuttal} and contrast it further to the assumptions made by \cite{sui2015safe} and \cite{tokmak2024pacsbo}.
In conclusion, we remove the \emph{a priori guess} on the RKHS norm by introducing Assumption~\ref{asm:random}, enabling us to incorporate data into the RKHS norm bound. 
Hence, we can cover a \emph{rich set of functions} and adjust the bounds as we gather more data (see Figures~\ref{fig:random_RKHS_functions} and~\ref{fig:numerical_investigation}), yielding reliable bounds in practice with random RKHS functions from sup- or sub-RKHSs of the RKHS of the ground truth.

\section{CONCLUSIONS}\label{sec:conclusion}
We presented a novel safe BO algorithm that learns an over-estimation of the RKHS norm from data, including statistical guarantees.
With that, it lifts the assumption of popular safe BO algorithms of knowing a tight upper bound on the RKHS norm a priori.
We further proved safety of the developed safe BO algorithm with RKHS norm over-estimation. 
The proposed algorithm was extended with an adaptive notion of locality and, thus, improved exploration and scalability. 
We demonstrated the benefits of our algorithm compared to \safeopt\ in simulation and showed that it can successfully handle real-world experiments.
Although we integrated the RKHS norm over-estimation and the locality into \safeopt, both can equally be integrated into any modification or extension thereof.
More importantly, we expect applications of the RKHS norm over-estimation to go beyond safe BO and open avenues for more realistic guarantees in general kernel-based methods or for estimating \eg Lipschitz constants with theoretical guarantees.
Future work includes proving optimality of Algorithm~\ref{alg:local_safe_BO}, investigating regret bounds, and disentangling the constraints from the reward function.
\section*{Acknowledgments}
We thank Harsha V.\ Guda for technical support.
We also wish to acknowledge CSC – IT Center for Science, Finland, for computational resources. 
This research was partially supported by \emph{Kjell och Märta Beijer Foundation}.
\bibliography{sample_paper}
\clearpage

 \begin{enumerate}

 \item For all models and algorithms presented, check if you include:
 \begin{enumerate}
   \item A clear description of the mathematical setting, assumptions, algorithm, and/or model. [Yes]
   \item An analysis of the properties and complexity (time, space, sample size) of any algorithm. [Yes]
   \item (Optional) Anonymized source code, with specification of all dependencies, including external libraries. [Yes]
 \end{enumerate}

 \item For any theoretical claim, check if you include:
 \begin{enumerate}
   \item Statements of the full set of assumptions of all theoretical results. [Yes]
   \item Complete proofs of all theoretical results. [Yes]
   \item Clear explanations of any assumptions. [Yes]     
 \end{enumerate}

 \item For all figures and tables that present empirical results, check if you include:
 \begin{enumerate}
   \item The code, data, and instructions needed to reproduce the main experimental results (either in the supplemental material or as a URL). [Yes]
   \item All the training details (e.g., data splits, hyperparameters, how they were chosen). [Yes]
         \item A clear definition of the specific measure or statistics and error bars (e.g., with respect to the random seed after running experiments multiple times). [Yes]
         \item A description of the computing infrastructure used. (e.g., type of GPUs, internal cluster, or cloud provider). [Yes]
 \end{enumerate}

 \item If you are using existing assets (e.g., code, data, models) or curating/releasing new assets, check if you include:
 \begin{enumerate}
   \item Citations of the creator If your work uses existing assets. [Yes]
   \item The license information of the assets, if applicable. [Not Applicable]
   \item New assets either in the supplemental material or as a URL, if applicable. [Yes] 
   \item Information about consent from data providers/curators. [Not Applicable]
   \item Discussion of sensible content if applicable, e.g., personally identifiable information or offensive content. [Not Applicable]
 \end{enumerate}

 \item If you used crowdsourcing or conducted research with human subjects, check if you include:
 \begin{enumerate}
   \item The full text of instructions given to participants and screenshots. [Not Applicable]
   \item Descriptions of potential participant risks, with links to Institutional Review Board (IRB) approvals if applicable. [Not Applicable]
   \item The estimated hourly wage paid to participants and the total amount spent on participant compensation. [Not Applicable]
 \end{enumerate}

 \end{enumerate}


%
%





%

%

\onecolumn














\aistatstitle{Safe exploration in reproducing kernel Hilbert spaces: Appendix}
\appendix
\addtocontents{toc}{\protect\setcounter{tocdepth}{2}}
\tableofcontents
\clearpage
\section{DERIVATION OF THE CONFIDENCE INTERVALS~\eqref{eq:Q}}\label{app:abbasi}
In this section, we derive the confidence intervals that were initially presented in the dissertation by \cite{abbasi2013online}.
These data-dependent bounds have gained increasing interest, see \eg \cite{Fiedler2021Practical} or \cite{fiedler2024safety}.

Writing Theorem~3.11 and Remark~3.13 by \cite{abbasi2013online} using our notation yields
\begin{align*}
\lvert f(\cdot)-\mu_t(\cdot)\rvert \leq \|\mathfrak m\|_{\overline V_t^{-1}}\left(
\sigma\sqrt{
2\log\left( 
\frac{\det(I_t+ K_t\sigma)^{\frac{1}{2}}}
{\delta}
\right)
}
+\sqrt{\sigma}B_t
\right),
\end{align*}
with
\begin{align*}
\|\mathfrak m\|_{\overline V_t^{-1}} = \frac{1}{\sqrt{\sigma}} \sigma_t(\cdot),
\end{align*}
which gives
\begin{align*}
\lvert f(\cdot)-\mu_t(\cdot)\rvert &\leq \frac{1}{\sqrt{\sigma}}\left(
\sigma\sqrt{
2\log\left( 
\frac{\det(I_t+ K_t/\sigma)^{\frac{1}{2}}}
{\delta}
\right)
}
+\sqrt \sigma B_t
\right)\sigma_t(\cdot) \\
&= \left(
\sqrt{\sigma}\sqrt{
2\log\left( 
\frac{\det(I_t+ K_t/\sigma)^{\frac{1}{2}}}
{\delta}
\right)
}
+B_t
\right)
\sigma_t(\cdot) \\
&= 
\left(
\sqrt{
2\sigma \log\left( 
\frac{\det(I_t+K_t/\sigma)^{\frac{1}{2}}}
{\delta}
\right)
}
+B_t
\right)
\sigma_t(\cdot) \\
&= 
\left(
\sqrt{
2\sigma \log\left( 
\det(I_t+ K_t/\sigma)^{\frac{1}{2}} \right)
-2 \sigma\log(\delta)
}
+B_t
\right)
\sigma_t(\cdot) \\
&= 
\left(
\sqrt{
\sigma \log\left( 
\det(I_t+ K_t/\sigma) 
\right)
-2 \sigma\log(\delta)
}
+B_t
\right)
\sigma_t(\cdot).
\end{align*}
\hfill \qed

\section{DERIVATION OF THE RKHS NORM FORMULA} \label{app:RKHS_norm_derivation}
In this section, we derive the general formula of the RKHS norm, \ie 
\begin{align*}
    \|f\|_k^2 = \sum_{s=1}^\infty \sum_{t=1}^\infty \alpha_i \alpha_j k(x_s,x_t).
\end{align*}
Let $f\in H_k$. 
    Then, we can write 
    \begin{align*}
        f=\sum_{t=1}^\infty\alpha_tk(\cdot,x_t)
    \end{align*}
In Hilbert spaces, the norm is given by the square root of the inner product of the function.
    Hence, 
    \begin{align*}
        \|f\|_k^2=\langle f, f \rangle_k,
    \end{align*}
 where $\langle \cdot, \cdot \rangle_k$ denotes the inner product of two functions in the RKHS of kernel $k$.
    Therefore, we have 
    \begin{align*}
    \|f\|_k^2&=\langle \sum_{t=1}^\infty\alpha_tk(\cdot,x_t), \sum_{t=1}^\infty\alpha_tk(\cdot,x_t) \rangle_k \\
    &=
    \sum_{t=1}^\infty\alpha_tk(\cdot,x_t) \sum_{s=1}^\infty\alpha_sk(\cdot,x_s) \\
    &=
    \sum_{t=1}^\infty
    \sum_{s=1}^\infty
    \alpha_t
    \alpha_s
    k(\cdot,x_t) k(\cdot,x_s) \\
    &= 
        \sum_{t=1}^\infty
    \sum_{s=1}^\infty
    \alpha_t
    \alpha_s
    k(x_s, x_t),
    \end{align*}
where the last equality follows from the reproducing property of reproducing kernel Hilbert spaces.
The ``center'' points are the~$x_s$ (or~$x_t$) points in this sum.
\section{ADDITIONAL FIGURE FOR THE INTRODUCTORY EXAMPLE}\label{app:intro}


Figure~\ref{fig:intro_over-estimation} shows the effect of conducting safe BO with a too conservative upper bound on the RKHS norm.
\begin{figure}[H]
    \centering
    \input{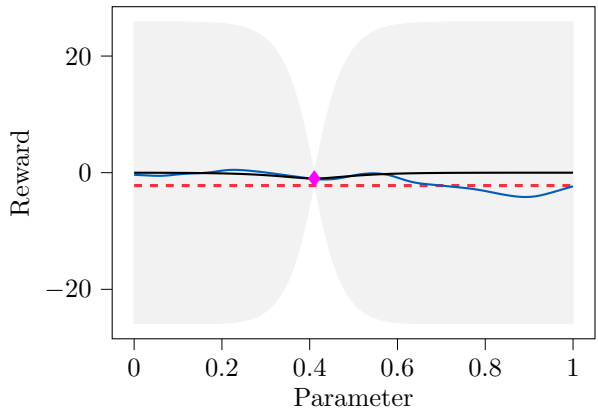}
    \caption{\emph{Safe BO corresponding to Figure~\ref{fig:SafeOpt intro}.}
    In this case, the guessed RKHS norm is five times the true RKHS norm, \ie the RKHS norm is conservatively over-estimated.
    The safe BO algorithm cannot sample any parameter since none is safe with high probability.
    Hence, a conservative over-estimation of the RKHS norm is undesirable.
    }
    \label{fig:intro_over-estimation}
\end{figure}
\section{ESTIMATING RKHS NORMS WITH RNNS}\label{app:RNN}

We use a custom RNN to process data from two distinct input sequences:
\emph{(i)} from the RKHS norm of the GP mean~$\mu_t$;
\emph{(ii)} from the reciprocal integral of the GP posterior variance~$\sigma_t^2$.
From these two sequences, the RNN extrapolates the unknown RKHS norm of the reward function~$\|f\|_k$.
For generating the training data and training the RNN, we used a cluster with \SI{60}{\giga\byte} 
RAM and 20 cores.

\paragraph{Architecture}
This model leverages two long-short-term memory RNN \citep{hochreiter1997long} branches with twenty hidden layers, respectively.
Moreover, each RNN branch contains two sigmoid and hyperbolic tangent activation functions, respectively.
We use this custom RNN setup to capture temporal dependencies within each input stream independently before merging their representations to produce unified predictions; see Figure~\ref{fig:RNN} for a schematic diagramm of the RNN.

\begin{figure}[h]
    \centering
    \usetikzlibrary{arrows.meta}
\begin{tikzpicture}

    \node (lstm1) [draw, rectangle, minimum width=40, minimum height=20] {LSTM 1};
    \node (lstm2) [draw, rectangle, minimum width=40, minimum height=20, below=of lstm1, yshift=-0.5cm] {LSTM 2};
    \node (concat) [draw, rectangle, minimum width=40, minimum height=20, right=of lstm1, xshift=0.5cm, yshift=-1.5cm] {Concatenate};
    \node (relu) [draw, rectangle, minimum width=40, minimum height=20, right=of concat, xshift=0.5cm] {ReLU};

    \draw[thick, -{Stealth[scale=1.5]}] (lstm1.east) -- (concat.west);
    \draw[thick, -{Stealth[scale=1.5]}] (lstm2.east) -- (concat.west);
    \draw[thick, -{Stealth[scale=1.5]}] (concat.east) -- (relu.west);


    \draw[thick, {Stealth[scale=1.5]}-] (lstm1.west) -- ++(-2,0) node[left] {GP mean};
    \draw[thick, {Stealth[scale=1.5]}-] (lstm2.west) -- ++(-2,0) node[left] {GP variance};
    \draw[thick, -{Stealth[scale=1.5]}] (relu.east) -- ++(2,0) node[right, align=center] {RKHS norm \\ estimate};
\end{tikzpicture}
    \caption{Schematic diagram of the used RNN.}
    \label{fig:RNN}
\end{figure}
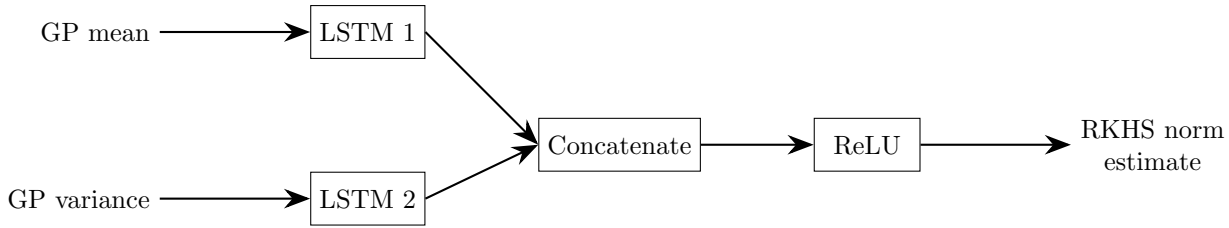

\paragraph{Training data}
Before training the RNN to estimate unknown RKHS norms~$\|f\|_k$, we require training data.
We generate training data by optimizing $10^3$ artificial RKHS functions~$g\in H_k$ using Algorithm~\ref{alg:local_safe_BO}.
To generate~$g$ and by executing Algorithm~\ref{alg:local_safe_BO}, we use the Matérn32 kernel with lengthscale~$\ell=0.1$.
We run Algorithm~\ref{alg:local_safe_BO} with $\delta=10^{-2}$, $\kappa=10^{-2}$, $\gamma=10^{-1}$, $\Delta=10^{-1}$, and~$N=3$ for 50 iterations.
To generate~$g$, we first sample the number of center points uniformly from $[600, 1000]$ and sample the center points~$x$ uniformly from~$\domain=[0,1]$.
Then, we sample~$\|g\|_k\in[0.5, 30]$ from a uniform distribution and scale the random coefficients~$\alpha$ to satisfy the pre-determined~$\|g\|_k$.
When executing Algorithm~\ref{alg:local_safe_BO}, we generate training data from each local object~$c\in \mathcal C_t$.
Hence, we require the corresponding RKHS norm~$\|g_c\|_k$ as the label, which is not directly inferred from the center points~$x$ and coefficients~$\alpha$ of the function~$g$.
Thus, we densely discretize the function~$g_c$ for any~$c\in C_t$ and any iteration~$t$, and compute a heuristic RKHS norm~$\|g_c\|_k$ using kernel interpolation; see \eg \cite{maddalena2021deterministic} for the computation of the RKHS norm of the interpolating function.

\paragraph{Performance}
The~$10^3$ functions~$g$ yield \num{280e3} training samples for the RNN.
We train the RNN with 100 epochs, a learning rate of $10^{-2}$, and the ADAM optimizer, which took around \SI{10}{\min}. 
We additionally preserved 20\% of validation data.
The root mean squared error on the validation data was approximately~\num{5e-3}.

\paragraph{Role of the RNN}
As mentioned in Section~\ref{sec:RKHS}, the RNN merely provides an additional layer of conservatism on the RKHS norm over-estimation and does not influence the provided theoretical guarantees.
However, it assists in accelerating Algorithm~\ref{alg:local_safe_BO}.
In Algorithm~\ref{alg:local_safe_BO}, we first loop through all local cubes~$\mathcal C_t$ and determine the most uncertain interesting parameter within each cube. 
Then, we conduct an experiment with the most uncertain interesting parameter among all local cubes.
Note that we only require PAC bounds, \ie guarantees on the RKHS norm over-estimation when conducting the experiment and not while looping through each local cube.
Therefore, to accelerate Algorithm~\ref{alg:local_safe_BO}, we loop through the local cubes and determine the uncertainties using only the RKHS norm estimation of the RNN since we do not conduct a safety-critical experiment yet.
Then, when having determined the most uncertain interesting parameter, we compute the PAC RKHS norm over-estimation and check whether the parameter remains safe with high probability before conducting the experiment.

\section{FURTHER ELABORATION ON ASSUMPTION~\ref{asm:random}}\label{app:assumption_rebuttal}

Assumption~\ref{asm:random} holds if the ground truth~$f$ and the random RKHS functions~$\rho_{t,j}$ are i.i.d.\ samples from the same---potentially unknown---probability space.
In the following, we demystify the assumption and contrast it to the assumption made in \safeopt, which is to assume that an upper bound on the RKHS norm is available a priori.

In concrete terms, our assumption restricts the complexity of the ground truth by requiring 
\begin{align*}
    f = \sum_{p=1}^{\hat N} \alpha_p k(x_p,\cdot).
\end{align*}
Since we compute random RKHS functions of the form
\begin{align*}
    \rho_{t,j} = \sum_{s=1}^{\hat N} \alpha_s k(x_s,\cdot), \quad t\geq 1, j\in\{1,\ldots,\mathrm m\},
\end{align*}
Assumption~\ref{asm:random} holds if~$\rho_{t,j}$ and~$f$ are i.i.d.\ samples from the same---potentially unknown---probability space.
We construct the random RKHS functions by deterministically fixing the first~$t\ll\hat N$ coefficients~$\alpha_{i:t}$ and center points~$x_{1:t}$ by the interpolation property subject to~$\sigma$-sub-Gaussian measurement noise.
The remaining coefficients and center points are i.i.d.\ samples from some probability space $(\Omega, \mathcal F, \nu)$.
Therefore, Assumption~\ref{asm:random} holds if
\begin{align}\label{eq:our_restriction}
   f = \sum_{p=1}^{\hat N} \alpha_p k(x_p,\cdot), \quad  \alpha_p, x_p 
\begin{cases}
    \text{from interpolation property} &\text{if } p \in [1,t] \\
     \text{i.i.d.\ samples from}\; (\Omega, \mathcal F, \nu) &\text{if } p \in [t+1,\hat N].
    \end{cases}
\end{align}
 In our experiments, we sample~$\alpha_p \in [-\bar\alpha,\bar\alpha]=[-1,1]$ and~$x_s\in\domain=[0,1]$ from \emph{uniform} distributions, allowing for a worst-case RKHS norm of up to 500.
In comparison, \safeopt requires
\begin{align}\label{eq:safeopt_restriction}
   f = \sum_{p=1}^{\infty} \alpha_p k(x_p,\cdot), \quad  \|f\|_k = \sqrt{\sum_{p=1}^\infty\sum_{s=1}^\infty \alpha_p\alpha_sk(x_p,x_s)} \overset{!}{\leq} B,
\end{align}
where~$B$ is the a priori upper bound on the RKHS norm of the ground truth~$f$.

Essentially, both assumptions, \ie \eqref{eq:our_restriction} and~\eqref{eq:safeopt_restriction}, restrict the complexity of the ground truth by assuming sufficiently regular behavior.
Regularity assumptions are necessary as considering arbitrarily complex functions would not lead to any practical bounds.
In practice, we need to approximate the probability space $(\Omega, \mathcal F, \nu)$.
Consequently, \safeopt\ needs to approximate the probability space as well; the set of possible outcomes are the functions from that RKHS \emph{subject to the RKHS norm condition} $\|f\|_k\leq B$.
The probability distribution in this setting is nature's probability distribution, which is the classic frequentist setting.
In contrast to \safeopt, our assumption captures larger RKHS norms and systematically incorporates data instead of having a static a priori restriction on the functions.

We mention throughout the paper how, why, and when \safeopt's assumption breaks;
more importantly, when this a priori restriction of possible functions for the ground truth~$f$ leads to safety violations. 
To make our assumption work in practice, we choose~$\hat N$ and~$\bar\alpha$ in~\eqref{eq:our_restriction} large enough to cover a broad range of functions.
Heuristically, it is sensible to restrict ourselves to a pre-RKHS setting (\ie $\hat N<\infty$) since the bounded norm property of~$\|f\|_k<\infty$ ``implies that the coefficients~$\alpha_p$ decay sufficiently fast as~$p$ increases'' \citep{berkenkamp2023bayesian}.

\paragraph{Contrasting Assumption~\ref{asm:random} to Assumption~1 by \cite{tokmak2024pacsbo}}

Assumption~\ref{asm:random} essentially states that the random RKHS functions and the reward function are i.i.d.\ samples from the same---potentially unknown---probability space. %
In comparison, Assumption~1 in \cite{tokmak2024pacsbo} requires that the RKHS norms of the random RKHS functions over-estimate the RKHS norm of the reward function in expectation.
Hence, the connection between the random RKHS functions and the ground truth does not directly become obvious since it is unclear under which conditions on the ground truth and the random RKHS functions this assumption holds.
In contrast, our assumption imposes a direct connection between the random RKHS functions and the ground truth.

\section{PROOFS}\label{app:proof}
In this section, we provide the proofs of our theoretical contributions, which we presented in Section~\ref{sec:theory}.
For the reader's convenience, we will restate the mathematical claims before providing their proofs within each subsection.

\subsection{Proof of Theorem~\ref{th:RKHS_scenario}}\label{proof:RKHS_scenario}

\setcounter{theorem}{0} 
\begin{theorem} [RKHS norm over-estimation]
Given Assumptions~\ref{asm:chow} and~\ref{asm:random}, for any iteration~$t\geq 1$,~$\gamma,\kappa\in(0,1)$, and $m\in\mathbb{N}$ such that $(1-\gamma)^{m-1}(1+\gamma(m-1))\leq\kappa$, consider~$B_t$ returned by Algorithm~\ref{alg:PAC}.
With confidence at least~$1-\kappa$, we have~$B_t\geq \|f\|_k$ with probability at least~$1-\gamma$.
\end{theorem}

We prove the theorem by following a (sampling-and-discarding) scenario approach \citep{campi2011sampling,calafiore2006scenario}.
Consider any iteration~$t\geq 1$ and write the RKHS norm over-estimation as a constrained optimization problem
\begin{align} \label{eq:impossible}
\min_{B_t^*\in\mathbb{R}_{\geq B_t}} \nonumber &B_t^* \\
\text{subject to} \quad &B_t^* \geq \|f\|_k.
\end{align}
In this notation,~$B_t^*$ corresponds to the optimization variable and~$B_t$ to the value returned by the RNN. 
We could similarly consider the optimization domain~$\mathbb R_{\geq 0}$.
However, by lower-bounding~$B_t^*$ with the initial estimate obtained from the RNN, we introduce some conservatism.
Clearly, Problem~\eqref{eq:impossible} is not solvable since~$\|f\|_k$ is unknown.
Hence, we formulate the optimization problem using the scenario approach \citep{calafiore2006scenario} with~$m$ i.i.d.\ random RKHS functions~$\rho_{t,j}$:

\begin{align}\label{eq:scenario}
\min_{B_t^*\in\mathbb{R}_{\geq B_t}} &B_t^* \nonumber \\
\text{subject to} \quad &B_t^* \geq \|\rho_{t,j}\|_k \quad \forall j\in\{1,\ldots, m\}.
\end{align}

We can use a scenario approach~\eqref{eq:scenario} to tackle Problem~\eqref{eq:impossible} since the RKHS norms are i.i.d.\ random variables from the same probability space \citep{calafiore2006scenario}.
Specifically, by solving~\eqref{eq:scenario}, we obtain a solution that satisfies all~$m$ constraints, which, in return, yields a PAC solution for Problem~\eqref{eq:impossible}.
However, some of the random RKHS functions could be outliers with unreasonably high RKHS norms. 
To trade feasibility (constraint satisfaction with respect to all random RKHS functions) for performance (a smaller RKHS norm over-estimation), we follow a sampling-and-discarding scenario approach \citep{campi2011sampling}.
To this end, we formulate the following scalar optimization problem:
\begin{align} \label{eq:discard}
\min_{B_t^*\in\mathbb{R}_{\geq B_t}} &B_t^* \nonumber  \\
\text{subject to} \quad &B_t^* \geq \|\rho_{t,j}\|_k \quad \forall i\in\{1,\ldots, m-r\} \\
&B_t^* < \|\rho_{t,j}\|_k \quad \forall j\in\{m-r+1,\ldots, m\}, \nonumber
\end{align}
\ie the optimal solution violates~$r$ constraints corresponding to the~$r$ largest random RKHS norms.

We continue to map our problem to a sampling-and-discarding scenario approach, specifically to Theorem~2.1 by \cite{campi2011sampling}.
Consider the probability space~$(\mathbb R_{\geq 0}, \mathcal B(\mathbb R_{\geq 0}), \mathbb P)$.
The probability space with~$m$ scenarios can be written as~$(\mathbb{R}_{\geq 0}^m, \mathcal B(\mathbb R_{\geq 0}^m), \mathbb P^m)$, \ie a classic \emph{product probability space}, equivalent to the setting in \cite{romao2022exact}.

Before using Theorem~2.1 by \cite{campi2011sampling}, we have to satisfy the following conditions:
\begin{enumerate}[label=(C\arabic*)]
    \item The domain of the optimization problem is convex and closed. \label{cond:1}
    \item The objective function is convex. \label{cond:2}
    \item The feasible domain is convex and closed. \label{cond:extra}
    \item The optimization problem is feasible for~$m<\infty$ with a feasibility domain with nonempty interior and unique solution.
    \label{cond:3}
    \item The optimal solution violates all~$r$ discarded constraints almost surely. \label{cond:4}
\end{enumerate}
We continue the proof in three different cases.

\paragraph{Case~I, $B_t < \|\rho_{t,m}\|_k \land B_t \leq B_{t-1}$}
In this case, the RKHS norm estimation returned by the RNN is smaller than the largest random RKHS norm and smaller than the previous PAC RKHS norm over-estimation.
Condition~\ref{cond:1} is satisfied since~$\mathbb{R}_{\geq B_t}$ is convex and closed for any~$B_t \in \mathbb R$.
Condition~\ref{cond:2} directly follows from having a linear objective function.
Condition~\ref{cond:extra} holds since the feasible domain is~$[
\|\rho_{t,m-r}\|_k, \|\rho_{t,m-r+1}\|_k
)\subseteq \mathbb{R}_{\geq 0}$, with~$r$ computed in Algorithm~\ref{alg:PAC}.
Moreover, Problem~\eqref{eq:discard} is feasible for~$m<\infty$ with a feasibility domain with nonempty interior and unique solution~\ref{cond:3}.
In fact, the solution of~\eqref{eq:discard} is
\begin{align}\label{eq:sol}
    B_{t,m,r}^\star=\max\{\|\rho_{t,m-r}\|_k,B_t\},
\end{align}
explicitly denoting that the value depends on the number of scenarios~$m$ and the number of removed constraints~$r<m$.

We now prove claim~\ref{cond:4}, \ie that~$B_{t,m,r}^\star$ under-estimates the RKHS norms corresponding to~$j=m-r+1,\ldots,m$ in~\eqref{eq:discard} almost surely.
To this end, note that the RKHS norms are sorted in an ascending order and that~$\|\rho_{t,j}\|_k\neq\|\rho_{t,i}\|_k, i,j\in\{1,\ldots,m\}, i\neq j$ almost surely.
Since~$B_{t,m,r}^\star=\max\{\|\rho_{t,m-r}\|_k,B_t\}$ with $B_t < \|\rho_{t,m-r}\|_k$ by Algorithm~\ref{alg:PAC} and~$\|\rho_{t,m-r}\|_k<\|\rho_{t,j}\|_k$, $\forall j\in\{m-r+1,\ldots,m\}$ almost surely, the claim holds.
Hence, we can use the result of Theorem~2.1 in \cite{campi2011sampling}:
\begin{align}\label{eq:campi_bound}
    &\mathbb P^m\left[(\|\rho_{t,1}\|_k,\ldots,\|\rho_{t,m}\|_k)\in \mathbb{R}_{\geq 0}^m :
    \mathbb P\left[
    \|f\|_k \in \mathbb R_{\geq 0}{:}\; B^\star_{t,m,r} \geq \|f\|_k
    \right]\geq 1-\gamma
    \right]\nonumber\\
    &\geq 1- \sum_{i=0}^r {m\choose i}\gamma^i (1-\gamma)^{m-i}.
\end{align}
Inequality~\eqref{eq:campi_bound} provides PAC bounds on the constraint satisfaction for any unknown random variable from the same probability space. 
Therefore, it probabilistically quantifies the constraint satisfaction of the optimal solution of~\eqref{eq:discard} with respect to the unsolvable optimization problem~\eqref{eq:impossible}, where we upper-bound the unknown RKHS norm~$\|f\|_k$. 
Since Algorithm~\ref{alg:PAC} requires 
\begin{align*}
     \sum_{i=0}^r {m\choose i}\gamma^i (1-\gamma)^{m-i} \leq\kappa
\end{align*}
and sets~$B_t=\max\{\|\rho_{t,m-r}\|_k,B_t\}$,
we have 
\begin{align}\label{eq:PAC_bounds}
    &\mathbb P^m\left[(\|\rho_{t,1}\|_k,\ldots,\|\rho_{t,m}\|_k)\in \mathbb{R}_{\geq 0}^m :
    \mathbb P\left[
    \|f\|_k \in \mathbb R_{\geq 0}{:}\; B_t \geq \|f\|_k
    \right]\geq 1-\gamma
    \right]\geq 1- \kappa,
\end{align}
which concludes the proof for Case~I.

\paragraph{Case~II, $B_t \geq \|\rho_{t,m}\|_k \land B_t \leq B_{t-1}$}
In this case, the RKHS norm estimation returned by the RNN is larger than the largest random RKHS norm and smaller than the previous PAC RKHS norm over-estimation.
Then, we recover the classic scenario approach, \ie we satisfy all~$m$ constraints, which can also be seen as a sampling-and-discarding scenario approach with~$r=0$ discarded constraints in Problem~\eqref{eq:discard}.
Conditions~\ref{cond:1}-\ref{cond:3} are satisfied equivalently to Case~I, and Condition~\ref{cond:4} holds trivially since~$r=0$.
The optimal solution of Problem~\eqref{eq:discard} is given by
\begin{align*}
    B_{t,m,0}^\star=B_t
\end{align*}
and Algorithm~\ref{alg:PAC} returns~$B_t$ as the PAC RKHS norm over-estimation.

Note that we choose~$\gamma,m,\kappa$ such that~$(1-\gamma)^{m-1}(1+\gamma(m-1))\leq\kappa$ in Theorem~\ref{th:RKHS_scenario}.
Since
\begin{align}\label{eq:case2_kappa}
 \sum_{i=0}^0 {m\choose i}\gamma^i (1-\gamma)^{m-i} & \leq \sum_{i=0}^1 {m\choose i}\gamma^i (1-\gamma)^{m-i} \nonumber \\
 &= (1-\gamma)^{m-1}(1+\gamma(m-1))  \\
 &\leq \kappa, \nonumber
\end{align}
we can directly obtain PAC bounds for the optimal solution of the sampling-and-discarding scenario approach~\eqref{eq:discard} with~$r=0$.
Namely,
\begin{align*}
    &\mathbb P^m\left[(\|\rho_{t,1}\|_k,\ldots,\|\rho_{t,m}\|_k)\in \mathbb{R}_{\geq 0}^m :
    \mathbb P\left[
    \|f\|_k \in \mathbb R_{\geq 0}{:}\; B_t \geq \|f\|_k
    \right]\geq 1-\gamma
    \right] \\
    &\geq 1- \sum_{i=0}^0 {m\choose i}\gamma^i (1-\gamma)^{m-i}  \overset{\eqref{eq:case2_kappa}}{\geq} 1- \kappa,
\end{align*}
which concludes the proof for Case~II.

\paragraph{Case~III, $B_t>B_{t-1}$}
We now consider the case where the RKHS norm over-estimation at the previous iteration was tighter than the over-estimation at the current iteration.
In this case, we choose
\begin{align*}
    B_t = \min\{B_t,B_{t-1}\},
\end{align*}
see Algorithm~\ref{alg:PAC},
with~$B_0=\infty$ by convention.
The reason behind this choice is that if the estimation is PAC at iteration~$t-1$, it is again PAC at iteration~$t$.
\hfill \qed

\subsection{Proof of Corollary~\ref{co:RKHS_stop}}\label{proof:RKHS_stop}

\setcounter{corollary}{0} 
\begin{corollary}[Lifting Theorem~\ref{th:RKHS_scenario} to all iterations]
Under the hypotheses of Theorem~\ref{th:RKHS_scenario}, receive~$B_t$ from Algorithm~\ref{alg:PAC} at all iterations~$t$.
Then, with confidence at least~$1-\kappa$,~$B_t$ over-estimates the ground truth RKHS norm~$\|f\|_k$ jointly for all iterations~$t\geq 1$ with probability at least~$1-\gamma$.
\end{corollary}

Let~$\{B_t\}_{t=1}^T, \; T\in\mathbb{N}$ be the discrete-time stochastic process containing the RKHS norm over-estimations for each iteration~$t$.
Since we choose~$B_t=\min\{B_{t-1}, B_t\}$ in Algorithm~\ref{alg:PAC}, we have
\begin{align}\label{eq:decreasing_B}
    B_t \leq B_{t-1} \leq \ldots \leq B_1\quad \forall t\geq 1.
\end{align}
Moreover, let~$\{\mathfrak F_t\}_{t=1}^T$ be a filtration with~$\mathfrak F_t=\sigma(B_1,\ldots,B_t)$ the~$\sigma$-algebras.
Then, we have that~$B_t\in \mathfrak F_t$ and due to~\eqref{eq:decreasing_B}, $\mathbb E[B_t]\leq B_1 < \infty$.\footnote{%
Note that the expected value is with respect to the probability space $(\mathbb{R}_{\geq 0}, \mathcal B(\mathbb R_{\geq 0}), \mathbb P)$ (\ie with respect to the probability measure~$\mathbb P$) since the random variable~$B_t$ is defined on that probability space.
}
Moreover,
\begin{align*}
    \mathbb{E}[B_{t+1} \vert\mathfrak F_{t}] \leq B_t \leq B_1\quad \forall t\geq 1,
\end{align*}
follows from~\eqref{eq:decreasing_B},
\ie$\{B_t\}_{t=1}^T$ is a supermartingale with respect to the filtration~$\{\mathfrak F_t\}^{T}_{t=1}$ \cite[Section~4.2]{durrett2019probability}.
Therefore, we can use a stopping-time construction for (super)martingales as done in Theorem~1 by \cite{Abbasi2011Improved} and Theorem~1 by \cite{chowdhury2017kernelized}.

Let us define the bad event
\begin{align*}
    \mathscr B_t=\left\{\omega \in \Omega{:}\;
    B_t < \|f\|_k
    \right\}    
\end{align*}
as under-estimating the ground truth RKHS norm~$\|f\|_k$.
Let~$\tau^\prime$ be the first time when the bad event~$\mathscr B_t$ happens, \ie 
\begin{align*}
    \tau^\prime(\omega) \coloneqq \min\{t\geq 1{:}\; \omega \in \mathscr{B}_t\}
\end{align*} 
with~$\min\{\emptyset\}=\infty$ by convention.
Since
\begin{align*}
    \bigcup_{t\geq 1} \mathscr{B}_t = \{\omega \in \Omega{:}\; \tau^\prime(\omega) < \infty\},
\end{align*}
we have
\begin{align}\label{eq:bad_event_chain}
    \mathbb P[\cup_{t\geq 1}  \mathscr B_t] &= \mathbb P[\tau^\prime<\infty] \nonumber\\
    &= \mathbb P[B_t < \|f\|_k, \tau^\prime < \infty] \\
    &\leq \mathbb P[B_t < \|f\|_k]. \nonumber 
\end{align}
In Theorem~\ref{th:RKHS_scenario}, we proved that
\begin{align*}
    \mathbb P^m\left[(\|\rho_{t,1}\|_k,\ldots,\|\rho_{t,m}\|_k)\in \mathbb{R}_{\geq 0}^m :
    \mathbb P\left[
    \|f\|_k \in \mathbb R_{\geq 0}{:}\; B_t \geq \|f\|_k
    \right]\geq 1-\gamma
    \right]\geq 1- \kappa.
\end{align*}
for any (fixed)~$t\geq 1$.
Therefore, 
\begin{align*}
    \mathbb P^m\left[(\|\rho_{t,1}\|_k,\ldots,\|\rho_{t,m}\|_k)\in \mathbb{R}_{\geq 0}^m :
    \mathbb P\left[
    \|f\|_k \in \mathbb R_{\geq 0}{:}\; B_t \geq \|f\|_k
    \right]\leq \gamma
    \right]\geq 1- \kappa.
\end{align*} for any (fixed)~$t\geq 1$,
which with~\eqref{eq:bad_event_chain} implies that
the statement in Theorem~\ref{th:RKHS_scenario} now holds
holds \emph{jointly} for all~$t\geq 1$.
That is, lifting the statement to hold jointly for all iterations is to upper-bound the probability that the bad event~$\mathscr B_t$ happens, which we do in~\eqref{eq:bad_event_chain}.
This probability is upper-bounded by the probability of under-estimating the RKHS norm.
In words, with confidence at least~$1-\kappa$, the RKHS norm over-estimation holds jointly for all iterations with probability at least~$1-\gamma$, where the confidence and probability are stated with respect to the probability measures~$\mathbb P$ and~$\mathbb P^m$, respectively.
\hfill \qed

\subsection{Proof of Theorem~\ref{th:error_bound}}\label{proof:error_bound}

\begin{theorem}[Confidence intervals] 
Under the same hypotheses as those of Corollary~\ref{co:RKHS_stop},  let~$B_t$ be returned by Algorithm~\ref{alg:PAC} $\forall t\geq 1$ with~$\kappa,\gamma\in(0,1)$.
Moreover, define~$Q_t(a)$ as in~\eqref{eq:Q} with any~$\delta\in(0,1)$ and~$C_t\coloneqq C_{t-1}\cap Q_t$ with~$C_0=\mathbb R$.
Then, with confidence at least~$1-\kappa$, $f(a) \in  C_t(a)$ holds jointly for all~$a \in \domain$ and for all~$t\geq 1$ with probability at least~$(1-\gamma)(1-\delta)$.
\end{theorem}

First, we define the following events (the complementary event is denoted by the superscript $\C$):

$\mathfrak C_t$: It holds that $f(a) \in C_t(a)$ jointly for all $a\in \domain$ and for all~$t\geq 1$.\\
$\mathcal E_t$: It holds that $\|\epsilon_{1:t}\|_{((K_t+\sigma I_t)^{-1}+I_t)^{-1}}\leq 2\sigma^2\ln\left(\frac{\sqrt{\det(1+\sigma)I_t+K_t}}{\delta}
\right)$ jointly for all~$t\geq 1$ and for any~$\delta\in(0,1)$.\\
$\mathfrak Q_t$: It holds that $f(a) \in Q_t(a)$ jointly for all $a\in \domain$ and for all~$t\geq 1$.\\
$\mathfrak B_t$: It holds that $B_t \geq \|f\|_k$ jointly for all~$t\geq 1$.

The proof aims at providing a lower bound on the probability of occurrence of event~$\mathfrak C_t$.
We start by investigating the probability of the event~$\mathfrak Q_t$ from which we can directly infer the probability of~$\mathfrak C_t$.

The challenge in this proof is that state-of-the-art confidence intervals from \eg \cite{abbasi2013online} or \cite{chowdhury2017kernelized} consider the RKHS norm as a \emph{deterministic} object and, therefore, only have the sub-Gaussian measurement noise as the source of stochasticity.
In contrast, we over-estimate the RKHS norm, thus making it a random variable.
Therefore, we have two sources of uncertainty that are defined on two distinct probability spaces.

\fakepar{Uncertainty~\emph{(i)}}
From Corollary~\ref{co:RKHS_stop}, we have that
\begin{align*}
    \mathbb P_1^m\left[(\|\rho_{t,1}\|_k,\ldots,\|\rho_{t,m}\|_k)\in \mathbb{R}_{\geq 0}^m :
    \mathbb P_1\left[
    \|f\|_k \in \mathbb R_{\geq 0}{:}\; B_t \geq \|f\|_k
    \right]\geq 1-\gamma
    \right]\geq 1- \kappa.
\end{align*}
jointly for all~$t\geq 1$,
\ie
\begin{align*}
    \mathbb P_1^m\left[(\|\rho_{t,1}\|_k,\ldots,\|\rho_{t,m}\|_k)\in \mathbb{R}_{\geq 0}^m :
    \mathbb P_1\left[
    \mathfrak B_t
    \right]\geq 1-\gamma
    \right]\geq 1- \kappa.
\end{align*}
The bounds are derived with respect to the inner probability space $(\mathbb R_{\geq 0}, \mathcal B(\mathbb R_{\geq 0}), \mathbb P_1)$
and the outer product probability space
$(\mathbb R_{\geq 0}^m, \mathcal B(\mathbb R_{\geq 0}^m), \mathbb P_1^m)$.
The inner probability measure~$\mathbb P_1$ quantifies the uncertainty on the \emph{hypothesis} that the RKHS norm over-estimation is correct, while the outer probability measure~$\mathbb P_1^m$ quantifies the \emph{sampling-based} uncertainty.\footnote{%
The random variable~$B_t$ is computed using a sampling-based approach by sampling~$m$ i.i.d.\ random RKHS functions; see Theorem~\ref{th:RKHS_scenario}.
Since the generation of this ``training set'' is random, the resulting hypothesis is endowed with additional uncertainty, which requires us to introduce the outer layer of probability.
}

We map the (inner) probability space to a simpler and more interpretable but for our needs equivalent probability space.
Instead of working on the sample space~$\mathbb R_{\geq 0}$, we work with the introduced events~$\mathfrak B_t$ and~$\mathfrak B_t^\C$.
Note that the event-based sample space and the $\sigma$-algebra are instances from the original sample space and $\sigma$-algebra. 
The events show a more interpretable version of the original probability space.
However, since the events are equivalently represented within the old and new settings, we preserve the original probability measure~$\mathbb P_1$ and obtain the probability space 
\begin{align*}
(\{\mathfrak B, \mathfrak B^\mathrm{C}\}, 2^{\{\mathfrak B, \mathfrak B^\mathrm{C}\}}, \mathbb P_1),
\end{align*}
with $2^{\{\mathfrak B, \mathfrak B^\mathrm{C}\}}\coloneqq\{\emptyset, \mathfrak B, \mathfrak B^\mathrm{C}, \{\mathfrak B, \mathfrak B^\mathrm{C}\}\}$, \ie the $\sigma$-algebra is the power set of the sample space.
In this discrete $\sigma$-algebra, the probability measure is given by the tabular mapping
\begin{align*}
    \mathbb P_1[\{\mathfrak B, \mathfrak B^\mathrm{C}\}] &= 1 \\
    \mathbb P_1[\emptyset] &= 0 \\
    \mathbb P_1[\mathfrak B] &=  1-\gamma \\
    \mathbb P_1[\mathfrak B^\mathrm C] &= \gamma, 
\end{align*}
where the first two results follow from the definition of valid probability measures, the third equality follows from Corollary~\ref{co:RKHS_stop}, while the final equality follows from the fact that $\mathbb P_1[\{\mathfrak B, \mathfrak B^\mathrm{C}\}]=\mathbb P_1[\mathfrak B]+\mathbb P_1[\mathfrak B^\mathrm{C}]$ since~$\mathfrak B$ and~$\mathfrak B^\mathrm{C}$ are disjoint by construction.

\begin{remark}[Source of Uncertainty~\emph{(i)}]
    The uncertainty arises from the randomness of the random RKHS functions. 
    As described in Section~\ref{sec:RKHS}, the first~$t$ center points and coefficients are \emph{deterministic} and set given the collected data points by the interpolating property subject to $\sigma$-sub-Gaussian noise.
    The randomness is solely introduced by sampling the tail coefficients and tail center points from, \eg uniform distributions.
    Therefore, this uncertainty is purely epistemic.
\end{remark}

\fakepar{Uncertainty~\emph{(ii)}}
From the works of \cite{chowdhury2017kernelized} and \cite{abbasi2013online}, we have probabilistic confidence intervals with a deterministic upper bound on the RKHS norm.
Hence, the uncertainty arises solely from the sub-Gaussian measurement noise, \ie from the probability of the occurrence of event~$\mathcal E_t$.
We form an equivalent event-based probability space transformation based on event~$\mathcal E_t$ as we did for Uncertainty~\emph{(i)} instead of working on the original probability space.
The event-based probability space is given by
\begin{align*}
    ((\mathcal E_t, \mathcal E_t^\C), 2^{\{\mathcal E_t, \mathcal E_t^\C\}}, \mathbb P_2).
\end{align*}
The $\sigma$-algebra is again the power set of the sample space, \ie $2^{\{\mathcal E_t,\mathcal E_t^\mathfrak{C}\}}=\{\{\mathcal E_t, \mathcal E_t^\mathrm{C}\},\mathcal E_t, \mathcal E_t^\mathrm{C},\emptyset\}$.
We take the same probability measure as in the original probability space and write the probability measure as the tabular mapping
\begin{align*}
    \mathbb P_2[\{\mathcal E_t, \mathcal E_t^\mathrm{C}\}] &= 1 \\
    \mathbb P_2[\emptyset] &= 0 \\
    \mathbb P_2[\mathcal E_t] &=  1-\delta \\
    \mathbb P_2[\mathcal E_t^\mathrm C] &= \delta.
\end{align*}
The fact that~$\mathbb P_2[\mathcal E_t]=1-\delta$ is derived in Theorem~1 by \cite{chowdhury2017kernelized} and Theorem~3.4 by \cite{abbasi2013online}.
The uncertainty of this event is purely aleatoric. It arises from applying Markov's inequality \citep[Theorem~1.6.4.]{durrett2019probability} to probabilistically bound the norm of the accumulated measurement noise with respect to a positive definite matrix.

Since we want to combine Uncertainties~\emph{(i)} and~\emph{(ii)}, we extend~$\mathbb P_2[\mathcal E_t]$ with the outer probability~$\mathbb P_1^m$.
That is, we include the sampling-based probability uncertainty on the training set of the random RKHS functions into the uncertainty of the measurement noise of conducting experiments.
This extension is trivial and purely artificial since the generation of the random RKHS functions does not influence the measurement noise. 
Therefore, we can state
\begin{align*}
    \mathbb{P}_1^m\left[
(\|\rho_{t,1}\|_k,\ldots,\|\rho_{t,m}\|_k)\in \mathbb{R}_{\geq 0}^m:
\mathbb P_2[
\mathcal E_t
]\geq 1-\delta
    \right] = 1.
\end{align*}

\fakepar{Constructing a product probability space}
Uncertainty~\emph{(i)} is purely epistemic and Uncertainty~\emph{(ii)} is purely aleatoric and both uncertainties are independent from each other.
Therefore, we can create a \emph{product probability space} between both individual probability spaces to quantify the uncertainty on the confidence interval, \ie on events~$\mathfrak Q_t$ and~$\mathfrak C_t$ while treating the RKHS norm as a random variable.

First, we create a the unique product probability measure as described in, \eg Theorem 1.7.1 by \cite{durrett2019probability}.
The unique probability measure of the product probability space is given by 
\begin{align*}
    \mathbb P[(\mathfrak B_t, \mathcal E_t)] = \mathbb P_1[\mathfrak B_t] \cdot \mathbb P_2[\mathcal E_t],
\end{align*}
where we denote by~$\mathbb P[(\cdot, \cdot)]$ the probability measure of the product probability space,\footnote{%
The product probability space is naturally given by the triple $((\mathfrak B_t, \mathfrak B_t^\C)\times (\mathcal E_t, \mathcal E_t), 2^{\{\mathfrak B_t, \mathfrak B_t^\C\} \times (\mathcal E_t, \mathcal E_t^\C)}, \mathbb P).$}
which maps a tuple of orthogonal random variables to the probability of joint occurrence.
Since
\begin{align*}
\mathbb P[(\mathfrak B_t, \mathcal E_t)] = \mathbb P_1[\mathfrak B_t] \cdot \mathbb P_2[\mathcal E_t],
\end{align*}
we have that
\begin{align*}
   \mathbb P[(\mathfrak B_t, \mathcal E_t)] = (1-\gamma)(1-\delta). 
\end{align*}
Moreover, we embed the uncertainty of the hypothesis into the stochasticity of the sampling process of the random RKHS functions by enveloping the inner probabilistic statement with the outer probability given by the measure~$\mathbb P_1^m$.
Hence, in conclusion, we can write
\begin{align*}
    \mathbb P_1^m\left[(\|\rho_{t,1}\|_k,\ldots,\|\rho_{t,m}\|_k)\in \mathbb{R}_{\geq 0}^m :
    \mathbb P\left[
    (\mathfrak B_t, \mathcal E_t)
    \right]\geq (1-\gamma)(1-\delta)
    \right]&\geq (1- \kappa)\cdot 1 = 1-\kappa
\end{align*}
Now, note that the ground truth~$f(a)$ lies within the confidence interval~$Q_t(a)$ if~\emph{(i)} the RKHS norm over-estimation and~\emph{(ii)} the bound on the accumulated noise hold, \ie if the event tuple~$(\mathfrak B_t, \mathcal E_t)$ holds.
Therefore, event~$\mathfrak Q_t$ is \emph{equivalent} to event $(\mathfrak B_t, \mathcal E_t)$ and we can write
\begin{align*}
    \mathbb P_1^m\left[(\|\rho_{t,1}\|_k,\ldots,\|\rho_{t,m}\|_k)\in \mathbb{R}_{\geq 0}^m :
    \mathbb P\left[
    \mathfrak Q_t
    \right]\geq (1-\gamma)(1-\delta)
    \right]&\geq 1-\kappa.
\end{align*}
In words, with confidence at least $1-\kappa$, the hypothesis that the ground truth lies within the confidence intervals with with the measurement noise \emph{and} the RKHS norm as random variables holds with probability at least $(1-\gamma)(1-\delta)$.
Finally, from Corollary~7.1 by \cite{berkenkamp2023bayesian}, we have that
\begin{align*}
    \mathbb P[\mathfrak C_t] \equiv \mathbb P[\mathfrak Q_t]
\end{align*}
and, therefore,
\begin{align*}
    \mathbb P_1^m\left[(\|\rho_{t,1}\|_k,\ldots,\|\rho_{t,m}\|_k)\in \mathbb{R}_{\geq 0}^m :
    \mathbb P\left[
    \mathfrak C_t
    \right]\geq (1-\gamma)(1-\delta)
    \right]&\geq 1-\kappa.
\end{align*}

The confidence is stated with respect to probability measure~$\mathbb P_1^m$ that comprises the stochasticity of the random RKHS function generation, whereas the probability of the hypothesis is stated with respect to the product probability measure~$\mathbb P$.\hfill \qed

\subsection{Proof of Theorem~\ref{th:safety}}\label{proof:safety}

\begin{theorem}[Safety] 
Under the same hypotheses as those of Theorem~\ref{th:error_bound}, initialize Algorithm~\ref{alg:global_safe_BO} with a safe set~$S_0\neq\emptyset: f(a)\geq h \; \forall a \in S_0$.
Then, with confidence at least~$1-\kappa$,
$
f(a_t) \geq h
$
jointly $\forall t\geq 1$
with probability at least~$(1-\gamma)(1-\delta)$ when running Algorithm~\ref{alg:global_safe_BO}.
\end{theorem}

The proof is similar to the proofs of Theorem~1 and Lemma~11 by \cite{sui2015safe}. 
Specifically, we prove that we remain safe with high probability when only sampling within the set of safe samples.

\safeopt-like algorithms start with an initial safe set and extend the safe set by (probabilistically) lower-bounding the function values of inputs on the domain based on information of inputs that are already classified as safe.
This interpretation naturally requires a notion of continuity and regularity to infer the behavior of inputs based on the behavior of neighboring inputs. 
Therefore, \safeopt\ and many \safeopt-like algorithms require the Lipschitz constant as an additional parameter next to the RKHS norm to execute the algorithm.
In contrast, we present a Lipschitz-like continuity for RKHS functions for which we use the (semi)metric~\eqref{eq:semi-metric}. 

\begin{lemma}[RKHS-induced continuity]\cite[Proposition~3.1]{Fiedler2023Lipschitz} \label{le:cont}
Let all conditions in Theorem~\ref{th:safety} hold and let~$B_t$ be returned by Algorithm~\ref{alg:PAC}.
Then, jointly for all $a,a^\prime\in\domain$ and for all~$t\geq 1$, with confidence at least~$1-\kappa$,
\begin{align*}
 \lvert f(a)-f(a^\prime)\rvert \leq B_t d_k(a,a^\prime)
\end{align*}
with probability at least~$1-\gamma$.\footnote{%
The confidence is stated with respect to the probability measure~$\mathbb P_1^m$ whereas the probability is stated with respect to the probability measure~$\mathbb P_1$.}
\end{lemma}
\begin{proof}
With confidence~$1-\kappa$ and probability~$1-\gamma$,
    \begin{align*}
\lvert f(a) - f(a^\prime) \rvert &= \lvert \langle f, k(a, \cdot) - k(a^\prime, \cdot) \rangle_k \rvert \tag*{
\cite[Definition~4.18]{Steinwart2008SVM}} \\
&\leq \|f\|_k \sqrt{k(a,a) - k(a^\prime, a) - k(a, a^\prime) + k(a^\prime, a^\prime)} \tag*{Cauchy-Schwarz inequality}\\
&= \|f\|_k
d_k(a,a^\prime) \tag*{(Semi)metric~\eqref{eq:semi-metric}}\\
&\leq B_t d_k(a,a^\prime) \tag*{\text{Corollary}~\ref{co:RKHS_stop}},
\end{align*}
where $\langle f, g \rangle_k$ denotes the inner product between two functions in the RKHS of kernel~$k$.
Note that solely the last inequality introduces stochasticity and the previous steps hold deterministically.
\end{proof}

For each iteration~$t\geq 1$, we are only allowed to sample within the safe set~$S_t$~\eqref{eq:safe_set}.
The following lemma exploits the definition of the safe set~$S_t$ to prove that we can guarantee safety with high probably for all iterations when only sampling within~$S_t$.
\begin{lemma}\label{le:sui_PAC}
Under the same hypotheses of Theorem~\ref{th:safety}, with confidence at least~$1-\kappa$,
\begin{align*}
    \forall a \in S_t, \; f(a) \geq h \
\end{align*}
jointly for all iterations~$t\geq 1$ with probability of at least~$(1-\delta)(1-\gamma)$.
\end{lemma}
\begin{proof}
The lemma is akin to Lemma~11 by \cite{sui2015safe}.
However, we replace the assumption of knowing a true upper bound on the RKHS norm~$\|f\|_k$ with the PAC RKHS norm over-estimation received by Algorithm~\ref{alg:PAC}.
Furthermore, in contrast to \cite{sui2015safe}, we do not require the Lipschitz constant and prove safety with high probability by exploiting the RKHS norm induced continuity formulated in Lemma~\ref{le:cont} instead.

First, similar to the proof of Theorem~\ref{th:error_bound}, we introduce the following events:\\
$ \Sigma_t$: It holds that~$f(a)\geq h$ jointly for all $a  \in S_t$ and for all $t \geq 1$.\\
$\mathfrak C_t$: It holds that $f(a) \in C_t(a)$ jointly for all $a\in \domain$ and for all~$t\geq 1$.\\
$\mathcal E_t$: It holds that $\|\epsilon_{1:t}\|_{((K_t+\sigma I_t)^{-1}+I_t)^{-1}}\leq 2\sigma^2\ln\left(\frac{\sqrt{\det(1+\sigma)I_t+K_t}}{\delta}
\right)$ jointly for all~$t\geq 1$ and for any~$\delta\in(0,1)$.\\
$\mathfrak B_t$: It holds that $B_t \geq \|f\|_k$ jointly for all~$t\geq 1$.\\
$\mathfrak L_t$: It holds that $\lvert f(a)-f(a^\prime)\rvert \leq B_t d_k(a,a^\prime)$ jointly for all $a\in \domain$ and for all~$t\geq 1$.

First, we project the statement in Lemma~\ref{le:cont} onto the product probability space before providing a lower bound on the occurrence of the event~$\Sigma_t$.
From Lemma~\ref{le:cont}, we have that
\begin{align*}
    \mathbb P_1^m\left[(\|\rho_{t,1}\|_k,\ldots,\|\rho_{t,m}\|_k)\in \mathbb{R}_{\geq 0}^m:
    \mathbb P_1\left[
    \mathfrak L_t
    \right]\geq 1-\gamma
    \right]&\geq  1-\kappa.
\end{align*}
We now embed the inner probability from probability measure~$\mathbb P_1$ into the product probability measure~$\mathbb P$. 
Note that event~$\mathfrak L_t$ (on the probability space governed by the probability measure~$\mathbb P_1$) solely depends on the correctness of the RKHS norm over-estimation, \ie on the occurrence of event~$\mathfrak B_t$.
Therefore, with respect to the product probability measure~$\mathbb P$, event~$\mathfrak L_t$
is \emph{equivalent} to the union of the events $(\mathfrak B_t, \mathcal E_t)\cup (\mathfrak B_t, \mathcal E_t^\C)$ since it solely depends on the correctness of the RKHS norm over-estimation.
Therefore, we can write
\begin{align*}
    \mathbb P_1\left[
    \mathfrak L_t
    \right] \equiv \mathbb P[(\mathfrak B_t, \mathcal E_t^\C)\cup (\mathfrak B_t, \mathcal E_t^\C)]
    &= \mathbb P[(\mathfrak B_t, \mathcal E_t)] +  \mathbb P[(\mathfrak B_t, \mathcal E_t^\C)] \\
    &= (1-\gamma)(1-\delta) + (1-\gamma)\delta \\
    &= 1-\gamma,
\end{align*}
\ie we can write $\mathfrak L_t \coloneqq (\mathfrak B_t, \mathcal E_t)\cup (\mathfrak B_t, \mathcal E_t^\C).$
Therefore, the probability that the continuity statement from Lemma~\ref{le:cont} holds is given by
\begin{align*}
        \mathbb P_1^m\left[(\|\rho_{t,1}\|_k,\ldots,\|\rho_{t,m}\|_k)\in \mathbb{R}_{\geq 0}^m:
    \mathbb P\left[
    \mathfrak L_t
    \right]\geq 1-\gamma
    \right]&\geq  (1-\kappa).
\end{align*}

\begin{remark}[Introduced events and their role in the product probability space]\label{re:events}
    We did not explicitly introduce~$\mathfrak L_t$ into the $\sigma$-algebra of the product probability space, which is the power set of the sample set $(\mathfrak B_t, \mathfrak B_t^\C)\times (\mathcal E_t, \mathcal E_t^\C)$.
However, we showed that~$\mathfrak L_t$ can be rewritten as $(\mathfrak B_t, \mathcal E_t)\cup (\mathfrak B_t, \mathcal E_t^\C)$, thus making it an explicit a member of the~$\sigma$-algebra of the product probability space.
Clearly, $(\mathfrak B_t, \mathcal E_t)\cup (\mathfrak B_t, \mathcal E_t^\C)=\mathfrak B_t$, \ie $\mathfrak L_t \equiv \mathfrak B_t$.
This already became evident in the proof of Lemma~\ref{le:cont}, where the stochasticity of the continuity statement was solely introduced by the stochasticity on the RKHS norm over-estimation.
\end{remark}

We now move on to lower-bounding the probability of occurrence of event~$\Sigma_t$, which clearly proves the lemma.
Equivalent to the proof of Lemma~11 in \cite{sui2015safe}, we prove the lemma by induction.

Base case:
In the first iteration, we set~$S_1\equiv S_0$, see Algorithm~\ref{alg:acquisition}.
Hence, by assumption, for all $a \in S_1, f(a) \geq h$ holds deterministically.

Induction step:
Assume for some~$t\geq 2$, $f(a)\geq h,\; \forall a \in S_{t-1}$.
We show that $f(a)\geq h, \; \forall a\in S_{t-1}$ implies $f(a^\prime)\geq h, \; \forall a^\prime\in S_{t}$.
By the definition of the safe set~\eqref{eq:safe_set},  $\forall a^\prime\in S_t$,  $\exists a \in S_{t-1}$  such that
\begin{align*}
    h \leq \ell_t(a) - B_td_k(a,a^\prime)
\end{align*}
holds deterministically.
Moreover, we have shown in Theorem~\ref{th:error_bound}, that
    \begin{align*}
    \mathbb P_1^m\left[(\|\rho_{t,1}\|_k,\ldots,\|\rho_{t,m}\|_k)\in \mathbb{R}_{\geq 0}^m :
    \mathbb P\left[
   \mathfrak C_t
    \right]\geq (1-\gamma)(1-\delta)
    \right]&\geq 1- \kappa,
\end{align*}
which implies that\footnote{%
The implication follows from the fact that $\min C_t(a)\eqqcolon\ell_t(a)$.
}
    \begin{align*}
    \mathbb P_1^m\left[(\|\rho_{t,1}\|_k,\ldots,\|\rho_{t,m}\|_k)\in \mathbb{R}_{\geq 0}^m :
    \mathbb P\left[
    f(a) \geq \ell_t(a)
    \right]\geq (1-\gamma)(1-\delta)
    \right]&\geq 1- \kappa.
\end{align*}
Therefore,
    \begin{align*}
    \mathbb P_1^m\left[(\|\rho_{t,1}\|_k,\ldots,\|\rho_{t,m}\|_k)\in \mathbb{R}_{\geq 0}^m :
    \mathbb P\left[
    h \leq f(a)-B_t d_k(a,a^\prime)
    \right]\geq (1-\gamma)(1-\delta)
    \right]&\geq 1- \kappa.
\end{align*}
Finally, observe that event~$\mathfrak C_t$ implies event~$\mathfrak L_t$ (see Remark~\ref{re:events}).
Therefore,
    \begin{align*}
    \mathbb P_1^m\left[(\|\rho_{t,1}\|_k,\ldots,\|\rho_{t,m}\|_k)\in \mathbb{R}_{\geq 0}^m :
    \mathbb P\left[
    h \leq f(a^\prime)
    \right]\geq (1-\gamma)(1-\delta)
    \right]&\geq 1- \kappa
\end{align*}
by Lemma~\ref{le:cont}.


\end{proof}
Theorem~\ref{th:safety} follows directly from Lemma~\ref{le:sui_PAC}. \hfill \qed
\section{RKHS NORM OVER-ESTIMATION TIGHTNESS COMPARED TO \cite{tokmak2024pacsbo}}\label{app:tightness}
Figure~\ref{fig:tightness_PACSBO} shows the tightness of the RKHS norm over-estimation by showing the ratio~$\nicefrac{B_t}{\|f\|_k}$ over the number of iterations.
We see that the RKHS norm over-estimation using the scenario approach (left sub-figure, our work) yields \emph{significantly} tighter over-estimations than the work by \cite{tokmak2024pacsbo}, who directly use Hoeffding's inequality to obtain PAC bounds.
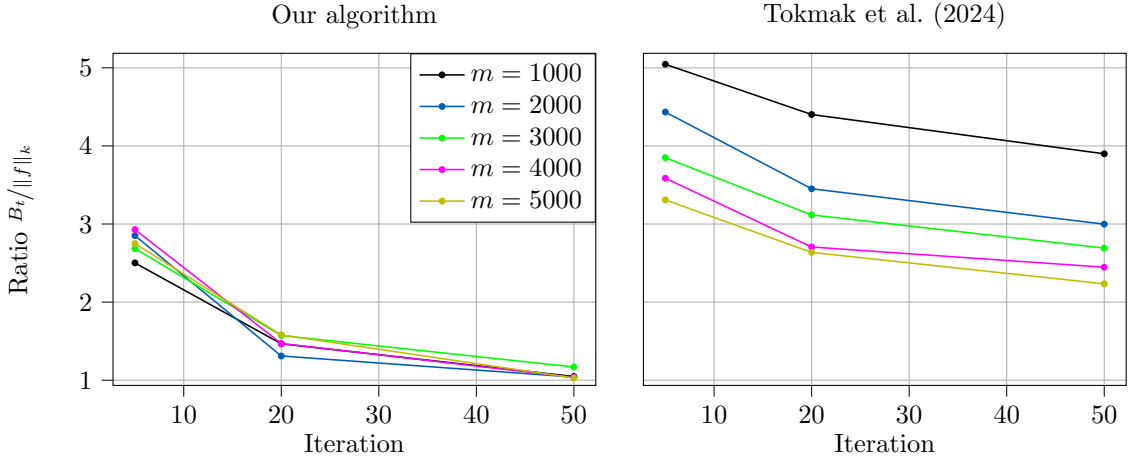
\begin{figure}
   \centering
\begin{tikzpicture}

\definecolor{darkgray176}{RGB}{176,176,176}
\definecolor{darkviolet1910191}{RGB}{191,0,191}
\definecolor{goldenrod1911910}{RGB}{191,191,0}
\definecolor{green01270}{RGB}{0,127,0}
\pgfplotsset{
every axis legend/.append style={
at={(1, 1)},
anchor=north east,
},
}
\begin{axis}[
tick align=outside,
tick pos=left,
title={Our algorithm},
x grid style={darkgray176},
xlabel={Iteration},
xmin=2.75, xmax=52.25,
xtick style={color=black},
y grid style={darkgray176},
ylabel={Ratio $\nicefrac{B_t}{\|f\|_k}$},
ymin=0.933425812721252, ymax=5.18607661724091,
ytick style={color=black},
width=8cm,
x grid style={darkgray176},
grid=both,
height=6cm
]
\addplot [semithick, black, mark=*, mark size=1, mark options={solid}]
table {%
5 2.50180377960205
20 1.46683855056763
50 1.0470739364624
};
\addplot [semithick, aaltoBlue, mark=*, mark size=1, mark options={solid}]
table {%
5 2.85113201141357
20 1.31017990112305
50 1.0402889251709
};
\addplot [semithick, green, mark=*, mark size=1, mark options={solid}]
table {%
5 2.68533611297607
20 1.57247476577759
50 1.16962976455688
};
\addplot [semithick, magenta, mark=*, mark size=1, mark options={solid}]
table {%
5 2.92869167327881
20 1.47010135650635
50 1.03612585067749
};
\addplot [semithick, goldenrod1911910, mark=*, mark size=1, mark options={solid}]
table {%
5 2.75022411346436
20 1.57946701049805
50 1.0284384727478
};
\legend{$m=1000$, $m=2000$, $m=3000$, $m=4000$, $m=5000$}
\end{axis}

\end{tikzpicture}
\begin{tikzpicture}

\definecolor{darkgray176}{RGB}{176,176,176}
\definecolor{darkviolet1910191}{RGB}{191,0,191}
\definecolor{goldenrod1911910}{RGB}{191,191,0}
\definecolor{green01270}{RGB}{0,127,0}

\begin{axis}[
tick align=outside,
tick pos=left,
title={Tokmak et al.\ (2024)},
x grid style={darkgray176},
xlabel={Iteration},
xmin=2.75, xmax=52.25,
xtick style={color=black},
y grid style={darkgray176},
ytick style={draw=none},
yticklabels={,,},
ymin=0.933425812721252, ymax=5.18607661724091,
ytick style={color=black},
width=8cm,
height=6cm,
x grid style={darkgray176},
grid=both
]
\addplot [semithick, black, mark=*, mark size=1, mark options={solid}]
table {%
5 5.04550838470459
20 4.40223598480225
50 3.89910244941711
};
\addplot [semithick, aaltoBlue, mark=*, mark size=1, mark options={solid}]
table {%
5 4.43380451202393
20 3.45140528678894
50 2.99823832511902
};
\addplot [semithick, green, mark=*, mark size=1, mark options={solid}]
table {%
5 3.84963226318359
20 3.11664152145386
50 2.69217872619629
};
\addplot [semithick, magenta, mark=*, mark size=1, mark options={solid}]
table {%
5 3.58673238754272
20 2.70576047897339
50 2.44706869125366
};
\addplot [semithick, goldenrod1911910, mark=*, mark size=1, mark options={solid}]
table {%
5 3.30864858627319
20 2.63648223876953
50 2.23414373397827
};
\end{axis}

\end{tikzpicture}
     \caption{Tightness of the RKHS norm over-estimation of our algorithm (left) compared to \cite{tokmak2024pacsbo} (right) when computing~$m$ random RKHS functions.}
     \label{fig:tightness_PACSBO}
\end{figure}
\section{ADAPTIVE NOTION OF LOCALITY} \label{app:adaptive_locality}

In Section~\ref{sec:locality}, we explain the motivation behind exploiting locality to reduce conservatism and, thus, improve exploration.
Figure~\ref{fig:locality_smoothness} shows a toy example in which working with the \emph{global} RKHS norm would result in unnecessarily conservative exploration in most parts of the domain.
This is because the sub-domain in the center has the largest slope and, thus, mainly contributes to increasing the (global) RKHS norm, whereas the sub-domains on the right- and left-hand side have \emph{significantly} smaller slopes and local RKHS norms.
These locally smaller RKHS norms naturally yield tighter confidence intervals when defining them on these local sub-domains.

Now, we contrast our \emph{adaptive} notion of locality to the locality introduced in \cite{tokmak2024pacsbo}.
As mentioned in Section~\ref{sec:locality}, we define sub-domains as local cubes around each sample~$a\in a_{1:t}$.
The size of these local cubes and the number of local cubes around each sample are hyperparameters. 
In addition to these sub-domains, we preserve the global domain.

In contrast, \cite{tokmak2024pacsbo} use  \emph{(i)}~the global domain, \emph{(ii)}~the convex hull of the samples, \emph{(iii)}~some intermediate domain between the convex hull and the global domain.
    A weakness of this approach is that it requires at least two samples to start exploration since the convex hull of a singleton set is purposeless in the considered setting.
    Furthermore, let~$S_0=\{0,1\}$ be the set of initial safe samples on the global domain~$\mathcal A=[0,1]$. 
    Then, the locality introduced by \cite{tokmak2024pacsbo} would again be impractical since the convex hull of samples is equal to the global domain, resulting in global exploration without a notion of locality.
    This design choice introduces limitations on the initial samples, restricting practical applicability.

    In our paper, singleton safe sets and points far apart do not negatively influence the notion of locality.
    Furthermore, we can additionally influence the size of the domain with the hyperparameter~$\Delta$.
    This hyperparameter is a major difference from \cite{tokmak2024pacsbo} and the main reason why our algorithm is significantly more scalable than \safeopt.
    Consider, \eg the six-dimensional half-cheetah environment, where we set $\Delta=0.05$.
    In \safeopt, or in the unfavorable local setting of \cite{tokmak2024pacsbo}, we would revert to~$\Delta=1$, which would give a Euclidean distance of~$\mathcal O(10^{-1})$ between the samples, whereas the distance in our case can be reduced by a factor of 200.
    This improves scalability, thus allowing for exploration in higher dimensions.

\begin{figure}
    \centering
\input{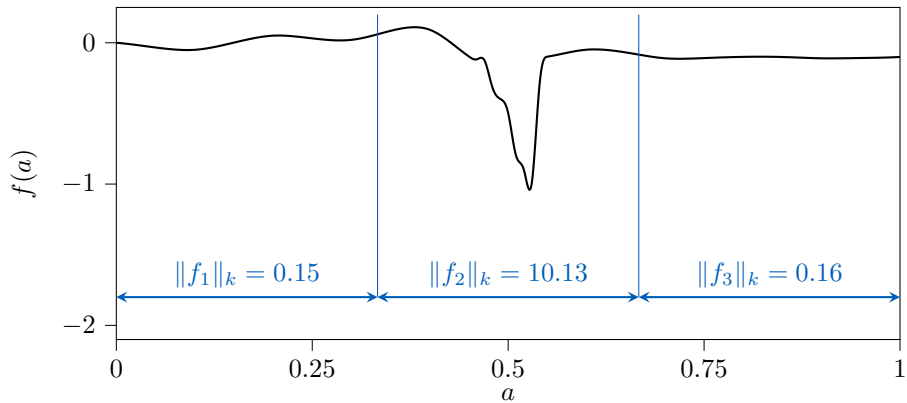}
\caption{\emph{General idea of locality.} We introduce a local interpretation of the RKHS norm and local sub-domains to exploit local ``smoothness'' and, thus, work with smaller RKHS norms on sub-domains with smaller slopes.}
    \label{fig:locality_smoothness}
\end{figure}

\section{REQUIRED COMPUTATION TIME FOR THE SCENARIO APPROACH}\label{app:scenario_time}
\begin{figure}
    \centering
\begin{tikzpicture}

\definecolor{crimson2143940}{RGB}{214,39,40}
\definecolor{darkgray176}{RGB}{176,176,176}
\definecolor{darkorange25512714}{RGB}{255,127,14}
\definecolor{forestgreen4416044}{RGB}{44,160,44}
\definecolor{steelblue31119180}{RGB}{31,119,180}

\pgfplotsset{
every axis legend/.append style={
at={(0, 0)},
anchor=south west,
legend columns=2
},
}
\begin{axis}[
tick align=outside,
tick pos=left,
legend style={nodes={scale=0.8}},
x grid style={darkgray176},
xtick style={color=black},
y grid style={darkgray176},
ymin=-5, ymax=47,
height=8cm,
grid=both,
width=12cm,
ytick style={color=black},
xlabel={Iteration},
ylabel={Time (\SI{}{\second})},
ytick={0, 20, 40}
]

\path [fill=aaltoBlue, fill opacity=0.1]
(axis cs:5,4.23205813270889)
--(axis cs:5,3.90662561553635)
--(axis cs:20,3.90933930124291)
--(axis cs:50,3.47098205322953)
--(axis cs:100,5.01331486581447)
--(axis cs:200,5.20138064881622)
--(axis cs:300,4.19649538259779)
--(axis cs:300,5.293666036317)
--(axis cs:300,5.293666036317)
--(axis cs:200,6.97096084097565)
--(axis cs:100,6.55898928762791)
--(axis cs:50,5.31811468367843)
--(axis cs:20,4.85340965543739)
--(axis cs:5,4.23205813270889)
--cycle;

\path [fill=darkorange25512714, fill opacity=0.1]
(axis cs:5,10.9675749470673)
--(axis cs:5,9.68008216123953)
--(axis cs:20,9.81869939918585)
--(axis cs:50,9.82342417841961)
--(axis cs:100,12.738153058968)
--(axis cs:200,12.7224686640653)
--(axis cs:300,9.45183713915094)
--(axis cs:300,11.7941501331116)
--(axis cs:300,11.7941501331116)
--(axis cs:200,14.0388911706057)
--(axis cs:100,14.6970326595684)
--(axis cs:50,12.4514519869704)
--(axis cs:20,11.9683925235742)
--(axis cs:5,10.9675749470673)
--cycle;

\path [fill=forestgreen4416044, fill opacity=0.1]
(axis cs:5,22.6149815257106)
--(axis cs:5,19.5689971749272)
--(axis cs:20,19.0733725700382)
--(axis cs:50,19.2945893190934)
--(axis cs:100,25.981862758749)
--(axis cs:200,25.2463405694075)
--(axis cs:300,19.2558916827379)
--(axis cs:300,22.7201111534895)
--(axis cs:300,22.7201111534895)
--(axis cs:200,28.5616637144975)
--(axis cs:100,27.5121150740452)
--(axis cs:50,21.3056266881392)
--(axis cs:20,21.8242956008908)
--(axis cs:5,22.6149815257106)
--cycle;

\path [fill=aaltoRed, fill opacity=0.1]
(axis cs:5,34.1373860679284)
--(axis cs:5,29.2726400532112)
--(axis cs:20,29.1583391314189)
--(axis cs:50,28.2154407854749)
--(axis cs:100,37.4763096714798)
--(axis cs:200,38.2871952310105)
--(axis cs:300,29.3602933249838)
--(axis cs:300,33.4285246529215)
--(axis cs:300,33.4285246529215)
--(axis cs:200,41.764780114601)
--(axis cs:100,45.7932561492141)
--(axis cs:50,36.1643607262897)
--(axis cs:20,34.2885557050069)
--(axis cs:5,34.1373860679284)
--cycle;

\addplot [semithick, aaltoBlue, mark=*, mark size=1, mark options={solid}]
table {%
5 4.06934187412262
20 4.38137447834015
50 4.39454836845398
100 5.78615207672119
200 6.08617074489594
300 4.7450807094574
};
\addplot [semithick, darkorange25512714, mark=*, mark size=1, mark options={solid}]
table {%
5 10.3238285541534
20 10.89354596138
50 11.137438082695
100 13.7175928592682
200 13.3806799173355
300 10.6229936361313
};
\addplot [semithick, forestgreen4416044, mark=*, mark size=1, mark options={solid}]
table {%
5 21.0919893503189
20 20.4488340854645
50 20.3001080036163
100 26.7469889163971
200 26.9040021419525
300 20.9880014181137
};
\addplot [semithick, aaltoRed, mark=*, mark size=1, mark options={solid}]
table {%
5 31.7050130605698
20 31.7234474182129
50 32.1899007558823
100 41.634782910347
200 40.0259876728058
300 31.3944089889526
};
\legend{$m=500$, $m=1000$, $m=2000$, $m=3000$}
\end{axis}

\end{tikzpicture}    \caption{\emph{Computation time for the scenario approach with~$m$ random RKHS functions.} The computation time is independent of the iteration and, hence, the amount of collected data.}
    \label{fig:scenario_time}
\end{figure}
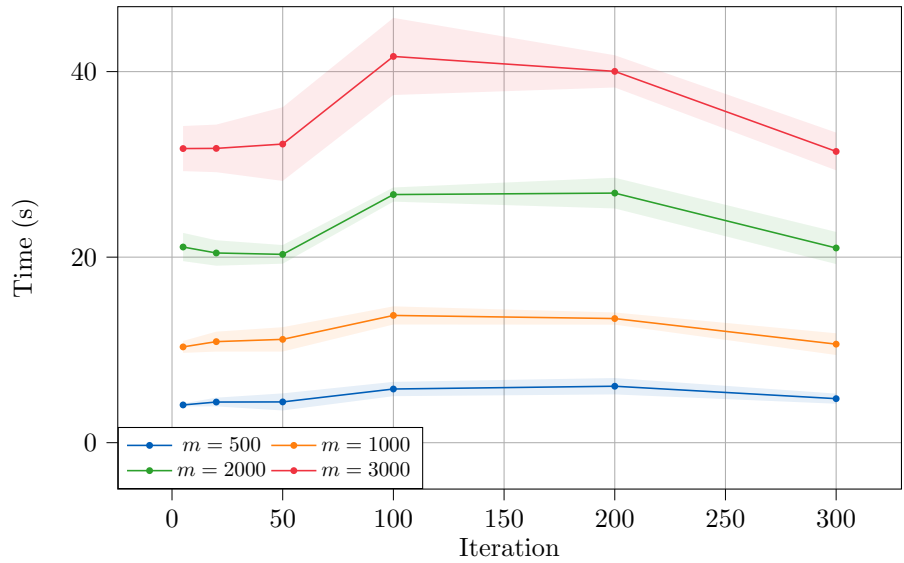
We investigate the required computation time for the scenario approach.
Additionally, we comment on the influence on the total computation time when integrating the scenario approach in \safeopt-type algorithms.

Figure~\ref{fig:scenario_time} shows the computation time of the scenario approach over the number of iterations for different~$m$.
The required time expectedly increases with the number of constraints~$m$, whereas a significant dependence of the computation time on the iteration cannot be observed.
This is in huge contrast to \safeopt.
The computation time of \safeopt\ is investigated in Figure~2 by \cite{baumann2023computationally}, where a significant increase is observable with the number of iterations.
Crucially, the influence of the scenario approach on the total computation time is insignificant; especially because our algorithm, akin to \safeopt, works in an episodic setting, where we do not have to meet real-time requirements.

\section{ABLATION STUDY FOR THE SCENARIO APPROACH}\label{app:ablation_scenario}

\begin{figure}
\centering
\begin{tikzpicture}

\definecolor{darkgray176}{RGB}{176,176,176}
\definecolor{steelblue31119180}{RGB}{31,119,180}

\begin{axis}[
tick align=outside,
tick pos=left,
title={\small{$m=2500$, $\kappa=0.001$}},
x grid style={darkgray176},
xlabel={$\gamma$},
xmin=-0.0039, xmax=0.1039,
xtick style={color=black},
y grid style={darkgray176},
ylabel={$r$},
ymin=-10.1, ymax=212.1,
width=5cm,
height=5cm,
x grid style={darkgray176},
ytick style={color=black},
grid=both]
\addplot [semithick, aaltoBlue]
table {%
0.001 0
0.002 0
0.003 0
0.004 1
0.005 2
0.006 4
0.007 5
0.008 7
0.009 8
0.01 10
0.011 12
0.012 14
0.013 15
0.014 17
0.015 19
0.016 21
0.017 23
0.018 25
0.019 27
0.02 29
0.021 31
0.022 33
0.023 35
0.024 37
0.025 39
0.026 41
0.027 43
0.028 45
0.029 47
0.03 49
0.031 51
0.032 53
0.033 55
0.034 57
0.035 59
0.036 62
0.037 64
0.038 66
0.039 68
0.04 70
0.041 72
0.042 74
0.043 77
0.044 79
0.045 81
0.046 83
0.047 85
0.048 87
0.049 89
0.05 92
0.051 94
0.052 96
0.053 98
0.054 100
0.055 103
0.056 105
0.057 107
0.058 109
0.059 111
0.06 114
0.061 116
0.062 118
0.063 120
0.064 122
0.065 125
0.066 127
0.067 129
0.068 131
0.069 134
0.07 136
0.071 138
0.072 140
0.073 143
0.074 145
0.075 147
0.076 149
0.077 152
0.078 154
0.079 156
0.08 158
0.081 161
0.082 163
0.083 165
0.084 167
0.085 170
0.086 172
0.087 174
0.088 176
0.089 179
0.09 181
0.091 183
0.092 186
0.093 188
0.094 190
0.095 192
0.096 195
0.097 197
0.098 199
0.099 202
};
\end{axis}

\end{tikzpicture}
\begin{tikzpicture}

\definecolor{darkgray176}{RGB}{176,176,176}
\definecolor{steelblue31119180}{RGB}{31,119,180}

\begin{axis}[
tick align=outside,
tick pos=left,
x grid style={darkgray176},
xlabel={$m$},
title={\small{$\gamma=0.01$, $\kappa=0.001$}},
xmin=0, xmax=7325,
xtick style={color=black},
y grid style={darkgray176},
ylabel={$r$},
width=5cm,
height=5cm,
x grid style={darkgray176},
grid=both,
ymin=-2.25, ymax=47.25,
ytick style={color=black}
]
\addplot [semithick, aaltoBlue]
table {%
500 0
1000 1
1500 4
2000 7
2500 10
3000 14
3500 17
4000 21
4500 25
5000 29
5500 33
6000 37
6500 41
7000 45
};
\end{axis}

\end{tikzpicture}
\begin{tikzpicture}

\definecolor{darkgray176}{RGB}{176,176,176}
\definecolor{steelblue31119180}{RGB}{31,119,180}

\begin{axis}[
log basis y={10},
tick align=outside,
tick pos=left,
x grid style={darkgray176},
grid=both,
xlabel={$m$},
ylabel style={yshift=1em},
xmin=0, xmax=7325,
xtick style={color=black},
y grid style={darkgray176},
ylabel={$\kappa$ (log)},
x grid style={darkgray176},
width=5cm,
height=5cm,
ymode=log,
ytick style={color=black},
ytick={1e-36,
1e-28,
1e-20,
1e-12,
1e-4
},
yticklabels={
  ${10^{-36}}$),
  $10^{-28}$,
  $10^{-20}$,
  $10^{-12}$,
  $10^{-4}$
},
title={\small{$r=0$, $\gamma=0.01$}}
]
\addplot [semithick, aaltoBlue]
table {%
500 0.0065704830424146
1000 4.31712474106579e-05
1500 2.83655949031613e-07
2000 1.86375660299223e-09
2500 1.22457811551487e-11
3000 8.04606974210249e-14
3500 5.28665647985697e-16
4000 3.47358867519715e-18
4500 2.28231554867063e-20
5000 1.49959156099795e-22
5500 9.8530409220851e-25
6000 6.47392382947773e-27
6500 4.25368067394672e-29
7000 2.79487367360137e-31
};
\end{axis}

\end{tikzpicture}
\caption{Influence of the hyperparameters of the scenario approach.}
\label{fig:ablation_scenario}
\end{figure}
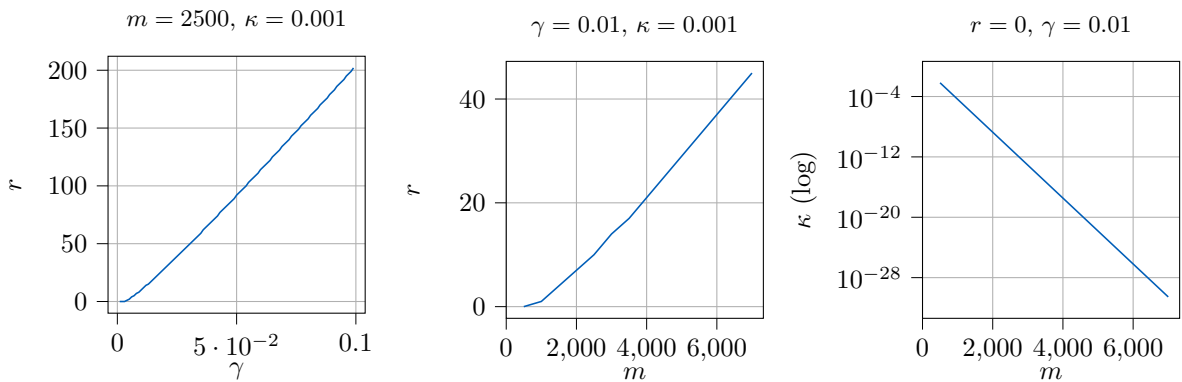

Figure~\ref{fig:ablation_scenario} shows an ablation study of the hyperparameters~$\gamma, m, r$, and~$\kappa$ for the scenario approach.
\emph{Left:} The number of constraints~$r$ that we can discard grows linearly with the decrease of the parameter $\gamma$ when keeping $m$ and $\kappa$ constant.
\emph{Center:} The number of constraints~$r$ that we can discard grows linearly with the number of total constraints~$m$ when keeping~$\gamma$ and~$\kappa$ constant.
\emph{Right:} The confidence parameter~$\kappa$ decreases exponentially with the number of total constraints~$m$ when keeping~$\gamma$ and~$r$ constant.

\section{ABLATION STUDY USING ALGORITHM~\ref{alg:local_safe_BO} WITH DIFFERENT KERNELS}\label{app:numerical_ablation}

\begin{figure}[h]
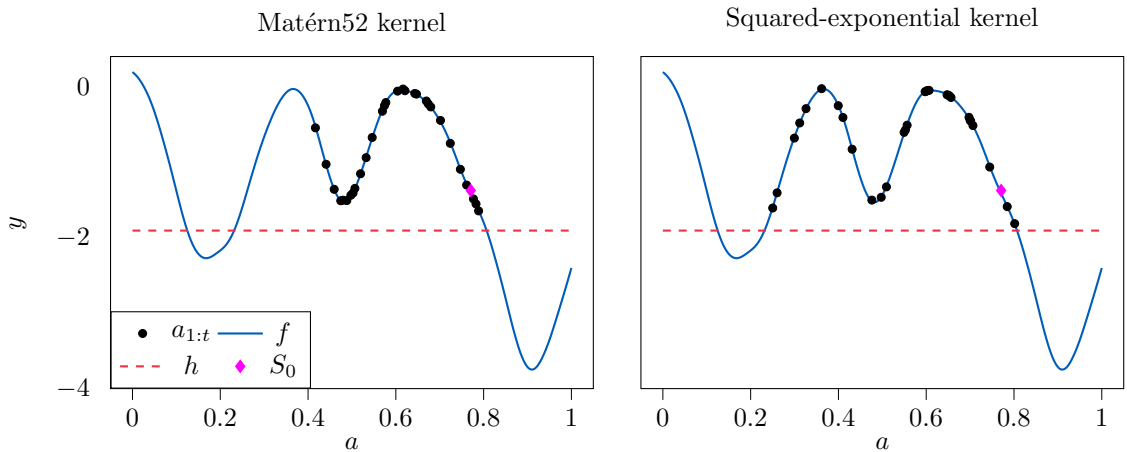

\centering
\input{final_figs/numerical_ablation_matern52}
\input{final_figs/numerical_ablation_RBF}
\caption{\emph{Performance of our safe BO algorithm using different kernels.} Even with a misspecified kernel, our algorithm explores the domain safely.}
\label{fig:numerical_ablation}
\end{figure}

In Figure~\ref{fig:numerical_ablation}, we perform an ablation study by conducting the numerical experiment from Section~\ref{sec:experiments} using Algorithm~\ref{alg:local_safe_BO} with different kernels.
Hence, the kernel used in Algorithm~\ref{alg:local_safe_BO} differs from the one with which the reward function is created.
Although we cannot provide any theoretical guarantees for this setting, Figure~\ref{fig:numerical_ablation} shows the successful deployment of our algorithm in a setting in which the kernel is misspecified.

\section{ABLATION STUDY USING ALGORITHM~\ref{alg:local_safe_BO} WITH DIFFERENT LOCALITY PARAMETERS}\label{app:locality_ablation}

\begin{figure}[h]
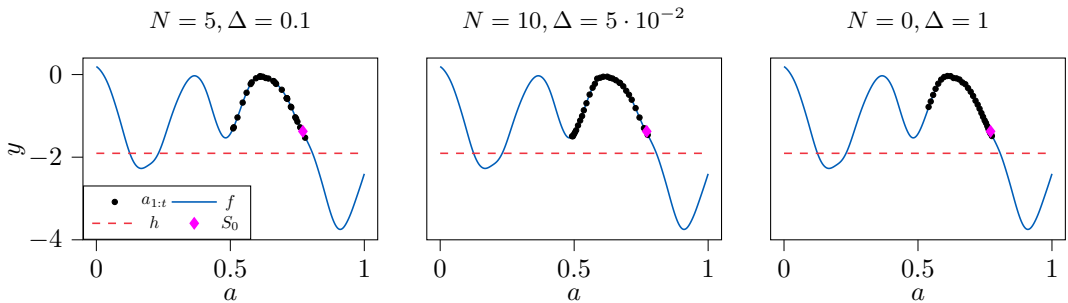

\centering
\input{final_figs/locality_ablation/original}
\input{final_figs/locality_ablation/delta_cube_005_N_10}
\input{final_figs/locality_ablation/delta_cube_1_N0}
\caption{Performance of our safe BO algorithm using different locality parameters~$N$ and~$\Delta$.}
\label{fig:locality_ablation}
\end{figure}

In Figure~\ref{fig:locality_ablation}, we perform an ablation study by conducting the numerical experiment from Section~\ref{sec:experiments} using Algorithm~\ref{alg:local_safe_BO} with different locality parameters~$N$ and~$\Delta$.
The left sub-figure is created using the original locality settings described in Section~\ref{sec:experiments}, \ie it is identical to the left sub-figure in Figure~\ref{fig:1D_toy}.
The sub-figure in the center is created using more and smaller local cubes.
Therefore, Algorithm~\ref{alg:local_safe_BO} has more options for local cubes to iterate through, which (slightly) improves exploration behavior.
The right sub-figure results from executing Algorithm~\ref{alg:local_safe_BO} without local cubes, \ie by only iterating through the global domain~$\domain$.
Therefore, we naturally recover Algorithm~\ref{alg:global_safe_BO}.
As expected and stated in Section~\ref{sec:locality}, we observe inferior exploration compared to the setting that actively exploits locality.
\section{SAFE RL POLICY OPTIMIZATION IN OPENAI GYM} \label{app:RL}
In this section, we provide further details on the RL benchmark simulations.
As discussed in Section~\ref{sec:experiments}, we trained the SAC algorithm \citep{haarnoja2018soft} in various OpenAI Gym environments \citep{brockman2016openai}, in particular, the mountain car, the cart-pole system, the swimmer, the lunar lander, the half-cheetah, and the ant. 
We then alter specific physical properties within each environment to imitate real-world experiments, in which we utilize our proposed algorithm and \safeopt\ to optimize an action bias matching the dimensionality of the action space.
We next state the remaining hyperparameters and detail how we alter the physical properties for the different environments.
We conducted the experiments on a cluster with \SI{60}{\giga\byte} RAM and 20 cores.

\paragraph{Mountain Car (1D)}
We set $N=3$, $\Delta=10^{-1}$, and discretize the environment with $10^3$ points.
For the imitated real experiments, we reduce the power of the car from $0.015$ to $0.013$.
The target is to reach the top of the mountain; any position before or behind the goal point at the end of an episode was considered unsafe.

\paragraph{Cart Pole (1D)}
We set $N=3$, $\Delta=10^{-1}$, and discretize the environment with $10^3$ points.
For the imitated real experiments, we change the pole length from $0.6$ to $0.8$.
The goal is to maintain the pole in an upright position; dropping the pole was considered unsafe.

\paragraph{Swimmer (2D)}
We set $N=5$, $\Delta=10^{-1}$, and discretize the environment with \num{5e2} points per dimension.
For the imitated real experiments, we change the lengths of the ``torso'' and ``back'' links from $0.$1 to $0.3$. 
The goal is to achieve forward movement of the swimmer; any backward movement was considered unsafe.

\paragraph{Lunar Lander (2D)}
We set $N=5$, $\Delta=10^{-1}$, and discretize the environment with \num{5e2} points per dimension.
For the imitated real experiments, we add wind of velocity \SI{3}{\meter\per\second}. 
The goal was for the lander to descend and come to a complete rest; any instance of the lander tipping over or crashing was considered unsafe.

\paragraph{Half Cheetah (6D)}
We set $N=10$, $\Delta=5\cdot 10^{-2}$, and discretize the environment with $8$ points per dimension.
For the imitated real experiments, we change the thickness of the back link from $0.046$ to $0.066$. 
The goal is to ensure forward movement without falling; any fall was considered unsafe.

\paragraph{Ant (8D)}
We set $N=10$, $\Delta=\num{5e-2}$, and discretize the environment with $5$ points per dimension.
For the imitated real experiments, we change the thickness of the leg joint from $0.08$ to $0.18$. 
The goal is to ensure forward movement without falling; any fall was considered unsafe.

\section{HARDWARE EXPERIMENT} \label{app:hardware}
We conducted the hardware experiment on an Ubuntu laptop with \SI{32}{\giga\byte} RAM and an Intel Core i7-12700H processor.
Figure~\ref{fig:furuta_appendix} shows the setup of the Furuta pendulum.

\begin{figure}
    \centering
    \includegraphics[width=6.75cm]{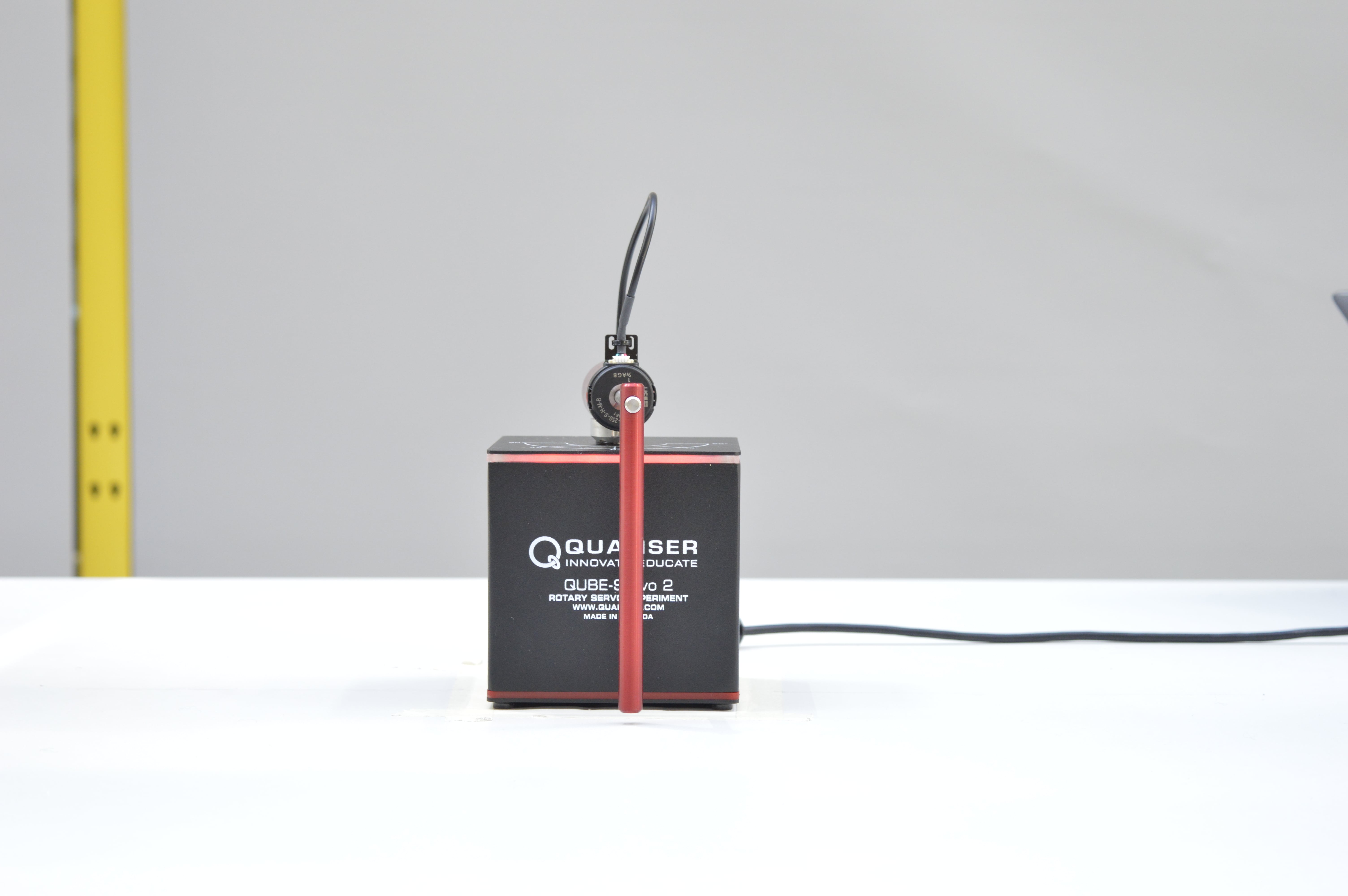}
    \includegraphics[width=6.75cm]{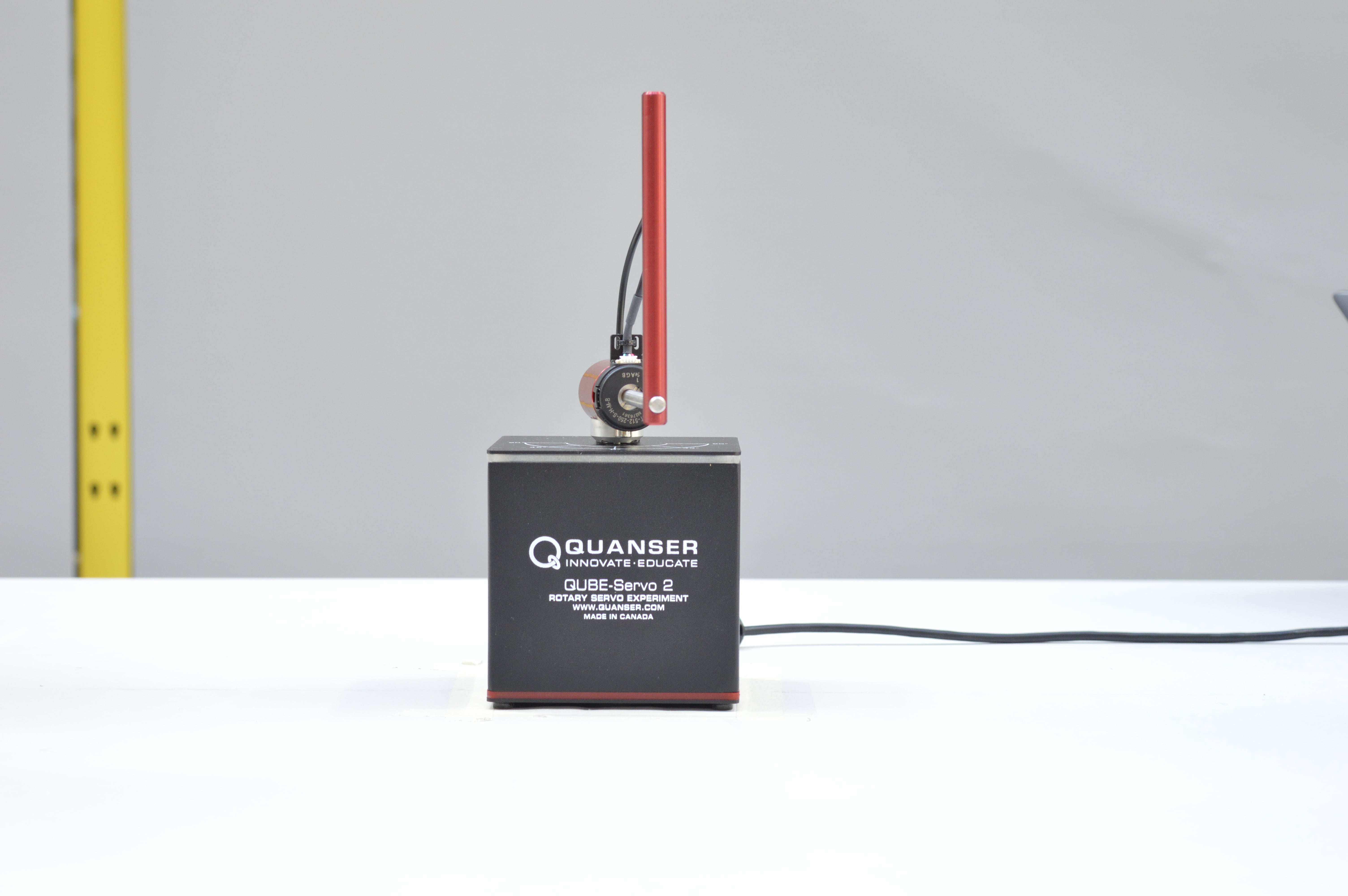}
    \caption{\emph{Hardware setup}.
    The Furuta pendulum starts from a downward position (left) and is swung upright.
    Then, we use a state-feedback controller to balance the pole (right).
    }
    \label{fig:furuta_appendix}
\end{figure}




\end{document}